\documentclass[twoside,11pt]{article}

\usepackage{jmlr2e}

\jmlrheading{1}{2000}{1-48}{4/00}{10/00}{Do-kyum Kim, Geoffrey M. Voelker and Lawrence K. Saul}

\ShortHeadings{Topic Modeling of Hierarchical Corpora}{Kim, Voelker and Saul}
\firstpageno{1}

\usepackage{amsmath}
\usepackage{bbm}
\usepackage{mathrsfs}

\newenvironment{packed_enum}{
\begin{enumerate}
  \setlength{\itemsep}{0pt}
  \setlength{\parskip}{0pt}
  \setlength{\parsep}{0pt}
}{\end{enumerate}}

\usepackage{balance}
\usepackage{multirow}
\usepackage{booktabs}

\usepackage{thmtools}
\usepackage{thm-restate}

\usepackage{hyperref}

\usepackage{algorithm}
\usepackage[noend]{algorithmic}

\begin{document}

\title{Topic Modeling of Hierarchical Corpora}

\author{\name Do-kyum Kim \email dok027@cs.ucsd.edu \\
\name Geoffrey M. Voelker \email voelker@cs.ucsd.edu \\
\name Lawrence K. Saul \email saul@cs.ucsd.edu \\
\addr Department of Computer Science and Engineering\\
University of California, San Diego\\
La Jolla, CA 92093-0404, USA}

\editor{ \hspace{1mm} }

\maketitle

\begin{abstract}
We study the problem of topic modeling in corpora whose documents are organized in a multi-level hierarchy.  We explore a parametric approach to this problem, assuming that the number of topics is known or can be estimated by cross-validation.  The models we consider can be viewed as special (finite-dimensional) instances of hierarchical Dirichlet processes (HDPs).  For these models we show that there exists a simple variational approximation for probabilistic inference.  The approximation relies on a previously unexploited inequality that handles the conditional dependence between Dirichlet latent variables in adjacent levels of the model's hierarchy.  
We compare our approach to existing implementations of nonparametric HDPs.  On several benchmarks we find that our approach is faster than Gibbs sampling and able to learn more predictive models than existing variational methods.  Finally, we demonstrate the large-scale viability of our approach on two newly available corpora from researchers in computer security---one with 350,000 documents and over 6,000 internal subcategories, the other with a five-level deep hierarchy.
\end{abstract}


\section{Introduction}

In the last decade, probabilistic topic models have emerged as a leading framework for analyzing and organizing large collections of text~\citep{blei_topic_2009}.  These models represent documents as ``bags of words" and explain frequent co-occurrences of words as evidence of topics that run throughout the corpus.  The first properly Bayesian topic model was latent Dirichlet allocation (LDA)~\citep{blei_latent_2003}.  A great deal of subsequent work has investigated hierarchical extensions of LDA, much of it stemming from interest in nonparametric Bayesian methods~\citep{Teh:2006dy}. In these models, topics are shared across different but related corpora (or across different parts of a single, larger corpus).  One challenge of topic models is that exact inference is intractable. Thus, it remains an active area of research to devise practical approximations for computing the statistics of their latent variables.



In this paper we are interested in the topic modeling of corpora whose documents are organized in a multi-level hierarchy.  Often such structure arises from prior knowledge of a corpus's subject matter and readership.  For example, news articles appear in different sections of the paper (e.g., business, politics), and these sections are sometimes further divided into subcategories (e.g., domestic, international).  Our goal is to explore the idea that prior knowledge of this form, though necessarily imperfect and incomplete, should inform the discovery of topics.  

We explore a parametric model of such corpora, assuming for simplicity that the number of topics is known or can be estimated by (say) cross-validation.  The models that we consider assign topic proportions to each node in a corpus's hierarchy---not only the leaf nodes that represent documents, but also the ancestor nodes that reflect higher-level categories.  Conditional dependence ensures that nearby nodes in the hierarchy have similar topic proportions.  In particular, the topic proportions of lower-level categories are on average the same as their parent categories.  However, useful variations naturally arise as one descends the hierarchy.  As we discuss later, these models can also be viewed as special (finite-dimensional) instances of hierarchical Dirichlet processes~(HDPs)~\citep{Teh:2006dy}.  

Our main contributions are two.  First, we devise a new variational approximation for inference in these models.  Based on a previously unexploited inequality, the approximation enables us to compute a rigorous lower bound on the likelihood in Bayesian networks where Dirichlet random variables appear as the children of other Dirichlet random variables.   We believe that this simple inequality will be of broad interest. 

Our second contribution is to demonstrate the large-scale viability of our approach.  Our interest in this subject arose from the need to analyze two sprawling, real-world corpora from the field of computer security.   The first is a seven-year collection of over 350,000 job postings from {\it Freelancer.com}, a popular Web site for crowdsourcing.  We view this corpus as a three-layer tree in which leaf nodes represent the site's job postings and interior nodes represent the active buyers (over 6,000 of them) on the site; see Figure~\ref{figure:freelancer_hierarchy}.  The second corpus is derived from the BlackHatWorld Internet forum, in which users create and extend threads in a deep, content-rich hierarchy of pre-defined subcategories; see Figure~\ref{figure:blackhatworld_hierarchy}.  Our results break new ground for hierarchical topic models in terms of both the breadth (i.e., number of interior nodes) and depth (i.e., number of levels) of the corpora that we consider.  Moreover, it is our experience that sampling-based approaches for HDPs~\citep{Teh:2006dy} do not easily scale to corpora of this size and depth, while other variational approaches~\citep{NIPS2007_763, NIPS2012_0208, NIPS2012_1251, Hoffman:2013tz} have not been demonstrated (or even fully developed) for hierarchies of this depth.

We have described our basic approach in previously published work~\citep{Kim_tiLDA_2013}.  The present paper includes many more details, including a complete proof of the key inequality for variational inference, a fuller description of the variational EM algorithm and its parallelization, and additional experimental results on the Freelancer and BlackHatWorld corpora.


 
The organization of this paper is as follows.  In Section~\ref{sec:model}, we describe our probabilistic models for hierarchical corpora and review related work.  In Section~\ref{sec:algorithm}, we develop the variational approximation for inference and parameter estimation in these models, and describe our parallel implementation of the inference and estimation procedures.  In Section~\ref{sec:experiment}, we evaluate our approach on several corpora and compare the results to existing implementations of HDPs.  Finally, in Section~\ref{sec:discuss}, we conclude and discuss possible extensions of interest.  The appendix contains a full proof of the key inequality for variational inference and a detailed derivation of the variational EM algorithm.




\begin{figure}[tb]
	\centering
	\includegraphics[width=4in]{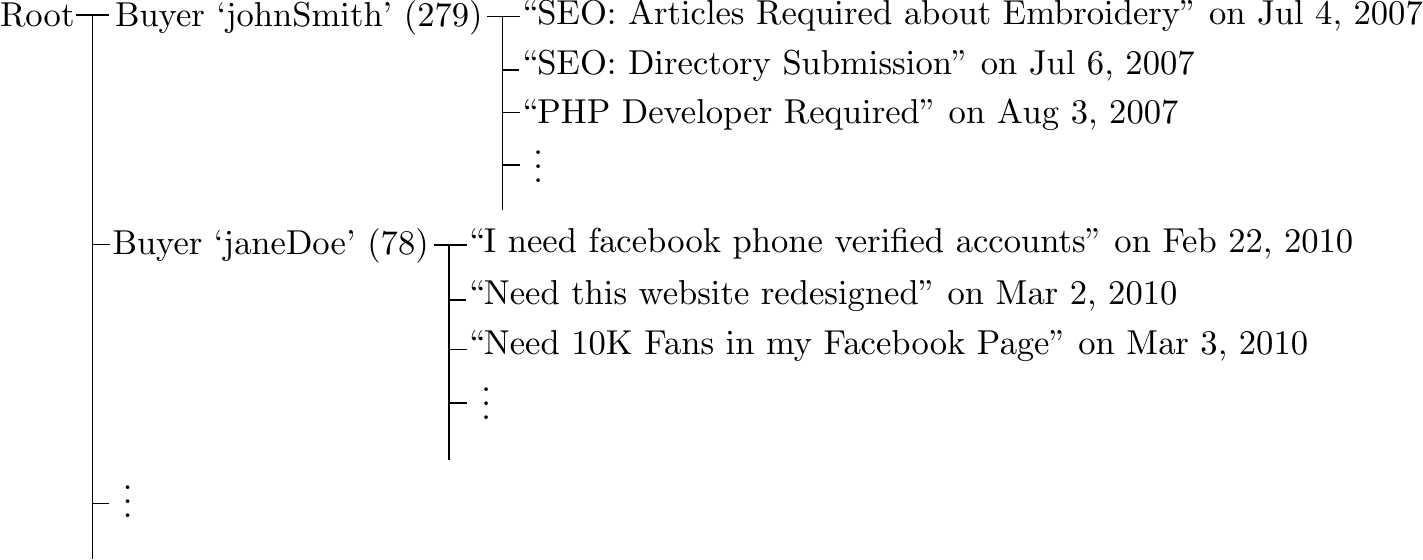}
	\caption{The hierarchy of buyers and job advertisements on Freelancer.com.  The number of ads per buyer is indicated in parentheses.  For brevity, only titles and dates of ads are shown.}
\label{figure:freelancer_hierarchy}
\end{figure}

\begin{figure}[tb]
	\centering
	\includegraphics[width=4in]{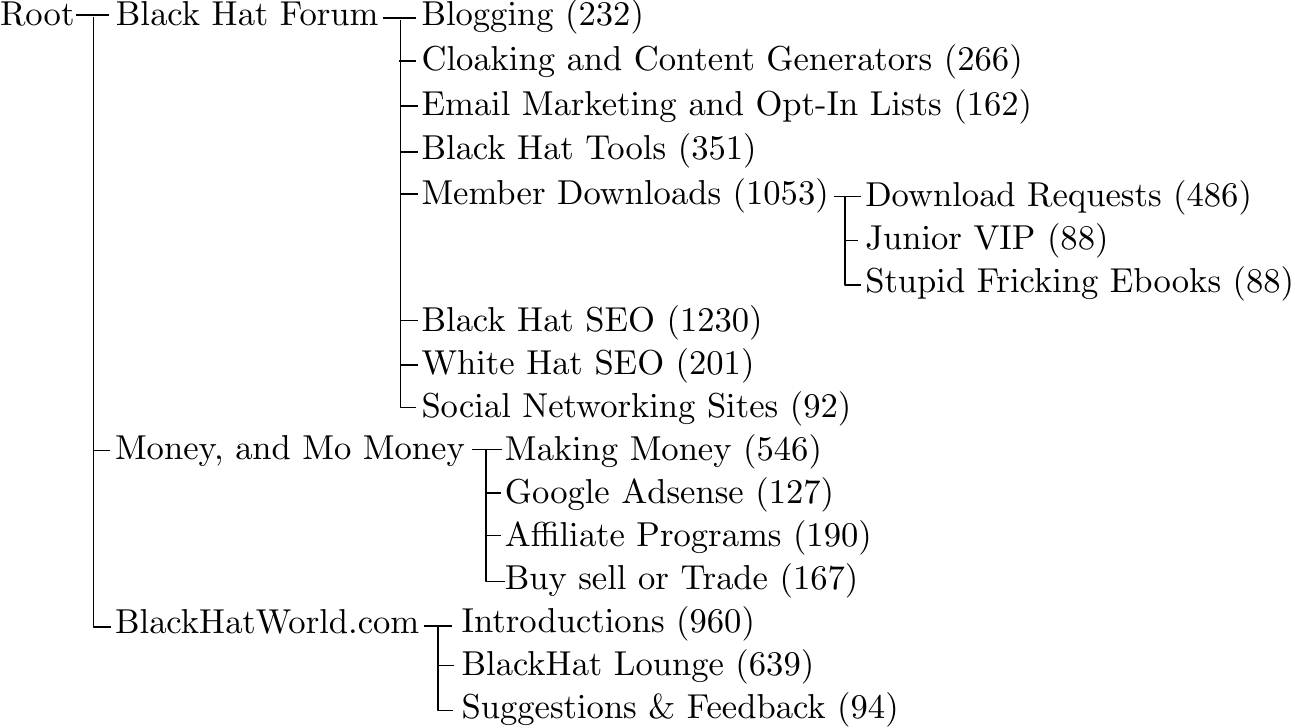}
	\caption{The hierarchy of subforums in the BlackHatWorld Internet forum.  The number of threads in each subforum is indicated in parentheses.}
\label{figure:blackhatworld_hierarchy}
\end{figure}

\section{Model and Related Work}
\label{sec:model}
Figures~\ref{figure:freelancer_hierarchy} and \ref{figure:blackhatworld_hierarchy} illustrate the types of structure we seek to model in hierarchical corpora.  This structure is most easily visualized as a tree in which the root node represents the corpus as a whole, the children of the root node represent top-level categories, the interior nodes represent subcategories of their parents, and the leaf nodes represent individual documents.  In this section we describe a probabilistic generative model for corpora of this form and discuss its relation to previous work in topic modeling. 



\subsection{Model}
Our model is essentially an extension of LDA to account for the tree structure in Figures~\ref{figure:freelancer_hierarchy} and~\ref{figure:blackhatworld_hierarchy}.  In LDA,
each document $d$ is modeled by topic proportions $\theta_d$, which are mixture weights over a finite set of $K$ topics.  In our approach, we model not only the documents in this way---the leaves of the tree---but also the categories and subcategories that appear at higher levels in the tree.  Thus for each (sub)category $t$, we model its topic proportions by a latent Dirichlet random variable $\theta_t$, and we associate one of these variables to each non-leaf node in the tree.
We use $\theta_0$ to denote the topic proportions of the root node in the tree (i.e. the corpus-wide topic proportions), and we sample these from a symmetric Dirichlet prior~$\gamma$.

The topic proportions of the corpus, its categories, subcategories, and documents are the latent Dirichlet variables in our model.  It remains to specify how these variables are related---in particular, how topic proportions are inherited from parent to child as one traverses the trees in Figures~\ref{figure:freelancer_hierarchy} and~\ref{figure:blackhatworld_hierarchy}.  We parameterize this conditional dependence by associating a (scalar) concentration parameter~$\alpha_t$ to each category~$t$.  The parameter $\alpha_t$ governs how closely the topic proportions of category $t$ are inherited by its subcategories and documents; in particular, small values of $\alpha_t$ allow for more variance, and large values for less.  More formally, let $\pi(t)$ denote the parent category of the category $t$.  Then we stipulate that:
\begin{equation}
\theta_t \sim \text{Dirichlet}\left( \alpha_{\pi(t)} \theta_{\pi(t)}\right).
\label{eq:inherit}	
\end{equation}
Likewise, our model assumes that documents inherit their topic proportions from parent categories in the same way.  In particular, we have:
\begin{equation}
\theta_d \sim \text{Dirichlet}\left( \alpha_{\pi(d)} \theta_{\pi(d)}\right),
\label{eq:inherit2}	
\end{equation}
where $\pi(d)$ in the above equation denotes the parent category of document $d$.  

The final assumption of our model is one of conditional independence: namely, that the topic proportions of subcategories are conditionally independent of their ``ancestral'' categories given the topic proportions of their parent categories.  With this assumption, we obtain the simple generative model of hierarchical corpora shown in Figure~\ref{figure:gen_process}. 
To generate an individual document, we begin by recursively sampling the topic proportions of its (sub)categories conditioned on those of their parent categories.  Finally, we sample the words of the document, conditioned on its topic proportions, in the same manner as~LDA.   In what follows we refer to this model as tree-informed LDA, or simply tiLDA.


\begin{figure}[tb]
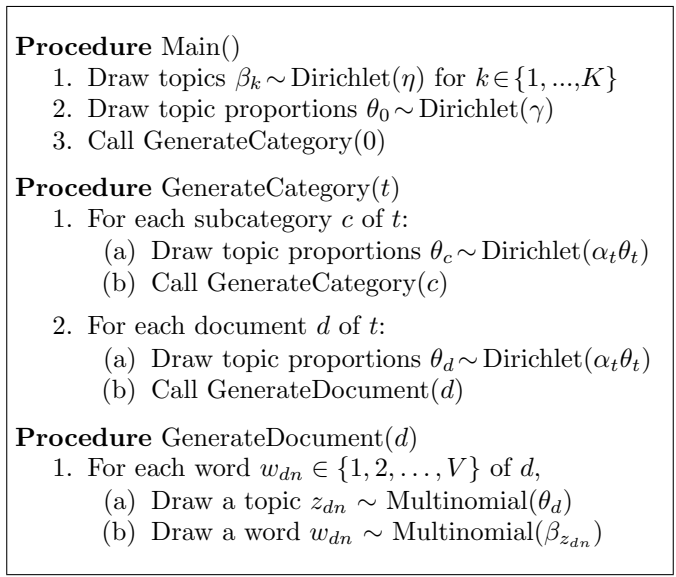
 \small
\begin{center}
\fbox{\parbox{3.4in}{
    \vspace{-1ex}
\begin{flushleft}
     {\bf Procedure} Main() 
     \vspace{-1.5ex}
     \begin{packed_enum}
			\item Draw topics $\beta_k\! \sim\!$ Dirichlet($\eta$) for $k\!\in\!\{1,...,\!K\}$
    \item Draw topic proportions $\theta_0\! \sim\!$ Dirichlet($\gamma$)
    \item Call GenerateCategory($0$)
	\end{packed_enum}

     {\bf Procedure} GenerateCategory($t$) 
          \vspace{-1.5ex}
	\begin{packed_enum}
    \item For each subcategory $c$ of $t$:
		\begin{packed_enum}
		     	\vspace{-1ex}
			\item Draw topic proportions $\theta_c\! \sim\!$ Dirichlet($\alpha_t \theta_t$)
			\item Call GenerateCategory($c$)
		\end{packed_enum}
    \item For each document $d$ of $t$:
		\vspace{-1ex}
		\begin{packed_enum}
			\item Draw topic proportions $\theta_d\! \sim\!$ Dirichlet($\alpha_t \theta_t$)
			\item Call GenerateDocument($d$)
		\end{packed_enum}
	\end{packed_enum}

	{\bf Procedure} GenerateDocument($d$) 
	     \vspace{-1.5ex}
         \begin{packed_enum}
	     \item For each word $w_{dn}\in\{1,2,\ldots,V\}$ of $d$,
			\vspace{-1ex}
			\begin{packed_enum}
				\item Draw a topic 
				$z_{dn}$ $\sim$ Multinomial($\theta_d$) 
				\item Draw a word $w_{dn}$ $\sim$ Multinomial($\beta_{z_{dn}}$)
			\end{packed_enum}
 	\end{packed_enum}
   \vspace{-2.5ex}
\end{flushleft}}}
\end{center}
\caption{The generative process of our topic model for hierarchical corpora.  The process begins in the Main procedure, sampling topic-word profiles and topic proportions from symmetric Dirichlet distributions.  Then it recursively executes the GenerateCategory procedure for each internal node of the corpus and the GenerateDocument procedure for each leaf node.}
\label{figure:gen_process}
\vspace{-2ex}
\end{figure}

In general, it is a bit unwieldy to depict the Bayesian network for topic models of this form.  However, a special case occurs when the corpus hierarchy has uniform depth---that is, when all documents are attached to subcategories at the same level.  Figure~\ref{figure:belief_network_full} shows the graphical model when all documents in the corpus are attached (for example) to third-level nodes.  

\subsection{Related Work}

Our model can be viewed as a generalization of certain previous approaches and a special instance of others.  Consider, for example, the special case of tiLDA for a ``flat" corpus, where all the documents are attached directly to its ``root."  This case of tiLDA corresponds to LDA with an asymmetric Dirichlet prior over topic proportions.  \citet{NIPS2009_0929} showed how to perform Gibbs sampling in such models and demonstrated their advantages over LDA with a symmetric Dirichlet prior.

Our approach also draws on inspiration from hierarchical Dirichlet processes (HDPs)~\citep{Teh:2006dy}.  In tiLDA, as in HDPs, the sample from one Dirichlet distribution serves as the base measure for another Dirichlet distribution.  HDPs are a nonparametric generalization of LDA in which the number of topics is potentially unbounded and can be learned from data.  We can view the generative model of tiLDA as a special case of multi-level HDPs whose base measure is finite (thus only allowing for a finite number of topics).  Though tiLDA does not possess the full richness of HDPs, our results will show that for some applications it is a compelling alternative.

Gibbs sampling is perhaps the most popular strategy for inference and learning in hierarchical topic models.  The seminal work by \citet{Teh:2006dy} developed a Gibbs sampler for HDPs of arbitrary depth and used it to learn a three-level hierarchical model of 160 papers from two distinct tracks of the NIPS conference.  \citet{Newman:2009uk} developed parallelized Gibbs samplers for two-level HDPs.  \citet{Du:2010ff} developed a collapsed Gibbs sampling algorithm for a three-level hierarchical model that is similar in spirit to ours.  One drawback to Gibbs sampling is that it can require long mixing times to obtain accurate results.  In general Gibbs sampling does not scale as well as variational inference~\citep{Blei:2014cp}.  


Many researchers have pursued variational inference in HDPs as a faster, cheaper alternative to Gibbs sampling.  \citet{NIPS2007_763} developed a framework for collapsed variational inference in two-level (but not arbitrarily deep) HDPs, and later \citet{Sato:2012ke} proposed a related but simplified approach. 
Yet another framework for variational inference was developed by~\citet{Wang:2011ux}, who achieved speedups with online updates.  While the first variational methods for HDPs truncated the number of possible topics, two recent papers have investigated online approaches with dynamically varying levels of truncation~\citep{NIPS2012_0208, NIPS2012_1251}.   There have been many successful applications of variational HDPs to large corpora; however, we are unaware of any actual applications to hierarchical corpora (i.e., involving HDPs that are three or more levels deep).   It seems fair to say that variational inference in nonparametric Bayesian models involves many complexities (e.g., auxiliary variables, stick-breaking constructions, truncation schemes) beyond those in parametric models.  We note that even for two-level HDPs, the variational approximations already require a good degree of cleverness (sometimes just to identify the latent variables).

The above considerations suggest regimes where an approach such as tiLDA may compete favorably with nonparametric HDPs.  In this paper, we are interested in topic models of large corpora with known hierarchical structure, sometimes many levels deep.  In addition, the corpora are static, not streaming; thus we are not attempting to model the introduction of new (or a potentially unbounded number of) topics over time.  We seek a model richer than LDA, one that can easily corporate prior knowledge in the form of Figures~\ref{figure:freelancer_hierarchy} and~\ref{figure:blackhatworld_hierarchy}, but with a minimum of additional complexity.  (Here it bears reminding that LDA---in its most basic form---still remains a wildly popular and successful model.)  We shall see that tiLDA fits the bill perfectly in this regime.  

Finally, our work assumes that the corpus has a {\it known} hierarchical structure.  A more ambitious direction is to relax this assumption and infer such structure (either among documents, or among topics) directly from observed data.  \citet{Adams_Tree_2010} proposed a nonparametric prior on tree structures of data and developed a Markov chain Monte Carlo algorithm for inference.  Other efforts in this direction include the nested Chinese restaurant process~\citep{Blei_Nested_2010} and the Pachinko allocation model~\citep{Li_Pachinko_2006}.  Perhaps the corpora in this paper can provide larger testbeds (with some notion of ground truth) to explore these more ambitious models.

\begin{figure}[tb]
	\centerline{\includegraphics[width=4in]{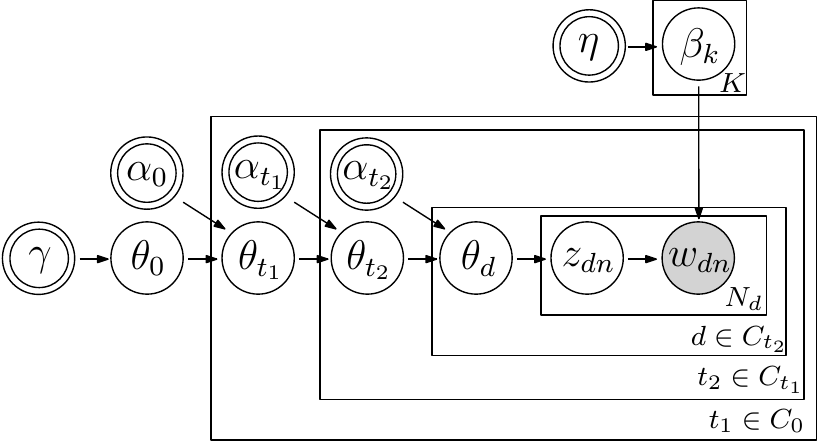}}
	\caption{Graphical model for tiLDA in which all documents of a hierarchical corpus are attached to third-level nodes. 
	Here $C_t$~denotes the set of indexes for the subcategories and documents of category~$t$, and $N_d$~denotes the length of the document~$d$.
	}
\label{figure:belief_network_full}
\end{figure}

\section{Algorithms}
\label{sec:algorithm}
In this section we develop the algorithms for inference and learning in tiLDA and also describe our parallelized implementation.

\subsection{Variational Inference}
\label{sec:inference}

The problem of inference in tiLDA is to compute the posterior distribution over the model's latent variables given the observed words in the corpus.  In tiLDA, the latent variables are the topic proportions $\theta_t$ of each category (or subcategory)~$t$, the topic proportions $\theta_d$ of each document $d$, the topic $z_{dn}$ associated with each word $w_{dn}$, and the multinomial parameters $\beta_k$ for each topic. Exact inference is not possible; approximations are required.  Here we pursue a variational method for approximate inference~\citep{jordan_introduction_1999} that generalizes earlier approaches to LDA~\citep{blei_latent_2003}.

The variational method is based on a parameterized approximation to the posterior distribution over the model's latent variables.  The approximation takes the fully factorized form:
\begin{align}
q(\theta, z, \beta | \nu, \rho, \lambda ) = \Big[ \prod_k\! q(\beta_k|\lambda_k) \Big] \Big[ \prod_t\! q(\theta_t|\nu_t) \Big]
							    \Big[ \prod_d\! q(\theta_d|\nu_d) \prod_n\! q(z_{dn}|\rho_{dn}) \Big],
\label{eq:variational}
\end{align}
where the parameters $\nu_t$, $\nu_d$, $\rho_{dn}$, and $\lambda_k$ are varied to make the approximation as accurate as possible.  The component distributions in this variational approximation are the exponential family distributions:
\begin{displaymath}
\begin{array}{ll}
\mbox{\hspace{1.2ex}}\theta_t \sim {\rm Dirichlet}(\nu_t),\quad &
\theta_d \sim {\rm Dirichlet}(\nu_d), \\
z_{dn} \sim {\rm Multinomial}(\rho_{dn}),\quad &
\beta_k  \sim {\rm Dirichlet}(\lambda_k).
\end{array}
\end{displaymath}
Figures~\ref{figure:belief_network_full} and \ref{figure:vars} contrast the graphical models for the true posterior and its variational approximation.

\begin{figure}[t]
\centerline{\includegraphics[width=4in]{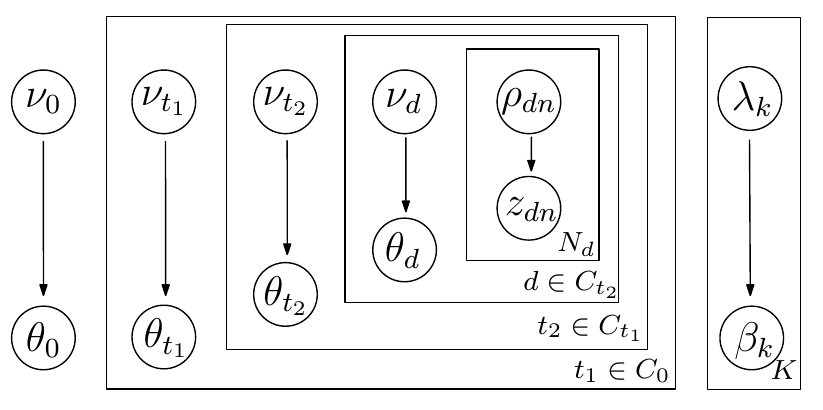}}
\caption{Variational approximation to the posterior distribution for the graphical model
in Figure~\ref{figure:belief_network_full}.}
	\label{figure:vars}
\end{figure}

The variational parameters $\nu_t$, $\nu_d$, $\rho_{dn}$, and $\lambda_k$ are found by attempting to minimize the Kullback-Leibler divergence between the approximation in eq.~(\ref{eq:variational}) and the
true posterior distribution of the model.  It can be shown that this is equivalent to maximizing
a lower bound $\mathscr{L} \leq \log p(w| \gamma, \alpha, \eta)$ on the marginal log-likelihood of the corpus.
The lower bound is given~by:
\begin{eqnarray}
\log p(w| \gamma, \alpha, \eta) 
  & = &  \log\! \int\!\! {\rm d}\beta\, {\rm d}\theta\, \sum_z\, p( \theta, z, w, \beta | \gamma, \alpha, \eta) \nonumber \\
  & = &  \log\! \int\!\! {\rm d}\beta\, {\rm d}\theta\, \sum_z\, q(\theta, z, \beta | \nu, \rho, \lambda ) 
  \left[\frac{p( \theta, z, w, \beta| \gamma, \alpha, \eta)}{q(\theta, z, \beta | \nu, \rho, \lambda ) }\right] \nonumber \\
  & \geq & \int\!\! {\rm d}\beta\, {\rm d}\theta\, \sum_z\, q(\theta, z, \beta | \nu, \rho, \lambda ) 
  \log \left[\frac{p( \theta, z, w, \beta| \gamma, \alpha, \eta)}{q(\theta, z, \beta | \nu, \rho, \lambda ) }\right]
\label{eqn:elbo1}
\end{eqnarray}
where the inequality on the third line follows from the concavity of the logarithm function and Jensen's inequality.  We may write the bound more compactly as:
\begin{equation}
	\mathscr{L}\ =\ {\rm E}_q \left[ \log p(\theta, z, w, \beta | \gamma, \alpha, \eta ) \right] + H(q), 
\label{eqn:elbo}
\end{equation}
where ${\rm E}_q$ denotes the expectation with respect to the variational distribution and $H(q)$ denotes its entropy.

So far we have developed the variational approximation for our model by following exactly the same approach used in LDA.  The lower bound in eq.~(\ref{eqn:elbo}), however, cannot be computed analytically, even for the simple factorized distribution in eq.~(\ref{eq:variational}).  In particular, new terms arise from the expectation ${\rm E}_q \left[ \log p( \theta, z, w, \beta | \gamma, \alpha, \eta ) \right]$ that are not present in the variational approximation for LDA.

Let us see where these terms arise.  Consider the model's {\it prior} distribution over latent topic proportions for each category $t$ and document $d$ in the corpus:
\begin{equation}
p( \theta | \alpha, \gamma )\ = \ p(\theta_0 | \gamma ) \prod_{t > 0}\, p(\theta_{t}| \alpha_{\pi(t)} \theta_{\pi(t)} ) 
\prod_{d}\, p(\theta_d| \alpha_{\pi(d)} \theta_{\pi(d)} ).
\label{eq:prior}
\end{equation}
In eq.~(\ref{eq:prior}), we have again used $\pi(t)$ and $\pi(d)$ to denote the parent categories of~$t$ and~$d$, respectively, in the tree.
Note that the last two terms in this prior distribution express conditional dependencies between Dirichlet variables at adjacent levels in the tree; as defined in eqs.~(\ref{eq:inherit}, \ref{eq:inherit2}).
In eq.~(\ref{eqn:elbo}), these dependencies give rise to a class of intractable averages such as ${\rm E}_q \left[  \log \Gamma( \alpha_{\pi(t)}\theta_{\pi(t)i}) \right]$.

These intractable averages stem from using a Dirichlet random variable as a base measure in another Dirichlet distribution.  This formulation also appears in the direct assignment representation of HDPs~\citep{liang-EtAl:2007, NIPS2012_1251}; in previous work, these intractable averages were sidestepped by substituting point estimates for the variational distributions on $\nu_t$.  Although such degenerate distributions can work in practice, this approach has two shortcomings: first, it does not account for potential variance in $\nu_t$, and second, it is prone to numerical error.

In our work, we exploit a novel bound to approximate the intractable averages that appear in models of this form.  To repeat, the problem arises from averages of the form ${\rm E}_q \left[  \log \Gamma( \alpha_{\pi(t)}\theta_{\pi(t)i}) \right]$, or more simply of the form ${\rm E}[\log\Gamma(x)]$ for some nonnegative random variable $x$.  We obtain a simple lower bound on such averages, ${\rm E}[\log\Gamma(x)] \geq \log\Gamma({\rm E}[x])$, by noting that $\log\Gamma(x)$ is a convex function of $x$ and appealing to Jensen's inequality.  But this simple bound goes in the wrong direction for our purposes; we need an {\it upper} bound on such averages to preserve a rigorous lower bound on the log-likelihood.  This is the crux of the problem.  The solution is contained in the following two lemmas.
\begin{restatable}{lem}{lemmaconcave}
	\label{lem:concave}
	Let $f(x) = \log \Gamma (x) + \log x - x \log x$ for $x\!>\!0$.  Then $f''(x)\!<\!0$ and $f(x)$ is a concave function of $x$.
\end{restatable}
\vspace{-2ex}
\begin{restatable}{lem}{lemmaloggammabound}
	\label{lem:logGammaBound}
	Let $x$ be a nonnegative random variable with bounded ${\rm E}[\log(1/x)] < \infty$.  Then:
	\begin{equation}
	{\rm E}[\log\Gamma(x)]
	\ \leq \ \log\Gamma({\rm E}[x]) + \log{\rm E}[x] - {\rm E}[\log x]  + {\rm E}[x\log x] - {\rm E}[x]\!\log {\rm E}[x]. 
	\label{eq:logGammaBound}
	\end{equation}	
\end{restatable}
\noindent
These lemmas are proved fully in Appendix~\ref{app:theorem}, while here we give just the briefest sketch.   The concavity of $f(x)$ in Lemma~\ref{lem:concave} is obtained by adding corrective terms to $\log\Gamma(x)$; these terms alter the function's behavior at both small and large values of $x$.  Likewise, the upper bound in Lemma~\ref{lem:logGammaBound} is obtained by applying Jensen's inequality to the concave function $f(x)$.

With these lemmas we are now equipped to approximate the intractable averages, of the form ${\rm E}_q\!\left[  \log \Gamma( \alpha_{\pi(t)}\theta_{\pi(t)i}) \right]$, that appear in eq.~(\ref{eqn:elbo}).  The result we need, for expected values of Dirichlet random variables, is given by the following theorem.

\begin{restatable}{thm}{theorembound}
\label{thm:bound}
Let $\theta\sim{\rm Dirichlet}(\nu)$, and let $\alpha>0$.  As shorthand, let $\nu_0 = \sum_i \nu_i$.  Then:
\begin{equation}
{\rm E}\left[ \log{\Gamma( \alpha \theta_{i})}\right] \ \leq \ \log\Gamma \! \left( \alpha {\rm E}[\theta_i]\right)
+ \alpha(1\!-\!{\rm E}[\theta_i])/\nu_0 + (1 \!-\! \alpha  {\rm E}[\theta_i]) \left[ \log  {\rm E}[\theta_i] \!+\! \Psi(\nu_0) \!-\! \Psi(\nu_i)\right],
\label{eqn:thm}
\end{equation}
where ${\rm E}[\theta_i] = \nu_i/\nu_0$ and $\Gamma(\cdot)$ and $\Psi(\cdot)$ are respectively the gamma and digamma functions.
\end{restatable} 

A proof of the theorem can be found in Appendix~\ref{app:theorem}.
Again we emphasize the direction of the bound.  The function $\log\Gamma(\cdot)$ is convex, and hence a naive application of Jensen's inequality to the left hand side of eq.~(\ref{eqn:thm}) yields the {\it lower} bound ${\rm E}[\log\Gamma( \alpha \theta_i )] \geq \log\Gamma(\alpha{\rm E}[\theta_i])$.  It is the {\it additional} terms on the right hand side of the equation that establish the theorem's {\it upper} bound.  The direction of inequality is crucial in the context of variational inference, where the upper bound is needed to maintain an overall lower bound on the log-likelihood. Thus it can be used to compute a looser (but still rigorous) lower bound $\mathscr{L}'\leq\mathscr{L}$ on the log-likelihood in terms of the model's variational parameters.  We shall see that this surrogate bound remains highly effective for inference and learning.  More details of this calculation can be found in Appendix~\ref{app:var_obj}.

We obtain the best approximation in the form of eq.~(\ref{eq:variational}) by maximizing $\mathscr{L}'$ with respect to the variational parameters $\nu$, $\rho$ and $\lambda$.
In practice, we perform the optimization by coordinate ascent in repeated bottom-up sweeps through the corpus hierarchy.
Each sweep begins by updating the parameters $\nu_d$ and $\rho_d$ attached to individual documents; these updates take essentially the same form as in LDA.
Then, once these parameters have converged, we turn to updating the variational parameters $\nu_t$ attached to different-level categories.
The bottom-up sweep continues through the different levels of the corpus until we reach the root of the corpus.
At the end of each sweep, we update the parameter $\lambda$ associated with topics.
Finally, the whole procedure repeats until $\mathscr{L}'$ converges.  For more details, we refer the reader to Appendices~\ref{app:max_doc_params}--\ref{app:max_topic_params}.

The inference procedure suggests an implementation by recursive calls along the corpus hierarchy as Algorithm 1.
The recursive function \textbf{OPT\_SUBTREE($t$)} performs optimizations on a subtree whose root node is $t$.
The function starts by initializing $\nu_t$ and then alternates between optimizing the children of $t$ and updating $\nu_t$ until $\mathscr{L}'$ converges.
The optimization on the children is done by recursively executing \textbf{OPT\_SUBTREE} for each child category of $t$ and \textbf{OPT\_DOCUMENT} for each child document of $t$.
The whole inference procedure is begun by calling \textbf{OPT\_SUBTREE} with the root node of the corpus.

\subsection{Variational Learning}

We can either fix the model parameters $\gamma$, $\alpha$ and $\eta$ or learn them from data.  For the latter, we use the lower bound from Section~\ref{sec:inference} as a surrogate for maximum likelihood estimation.  The variational EM algorithm alternates between computing the best factorized approximation in eq.~(\ref{eq:variational}) and updating the model parameters to maximize the lower bound~$\mathscr{L}'$.  The first of these steps is the variational E-step; the second is the variational M-step.   In the M-step we update the model parameters by block coordinate ascent.  In particular, we use Newton's method to update the concentration parameter~$\alpha_t$ associated to each category~$t$ as well as $\gamma$ and $\eta$ at the root of the corpus; for more details on these optimizations, we refer the reader to Appendix~\ref{app:estimate_params}.

It is useful to view the variational EM algorithm as a double-optimization over both the variational parameters (E-step) and the model parameters (M-step).  This view naturally suggests an interweaving of the two steps, and, in fact, this is how we implement the algorithm in practice; see Algorithm 1.

\begin{algorithm}[tb]
\caption{The variational EM algorithm for tiLDA. The algorithm begins in main, then invokes \textbf{OPT\_SUBTREE} recursively for each category.  At the deepest level of recursion, \textbf{OPT\_DOCUMENT} infers the hidden variables of documents given their words and prior on topic proportions (just as in LDA).} 
	\label{alg:main}
	\begin{algorithmic}[1]
	\MAIN {()}
		\STATE initialize $\gamma$, $\eta$ and $\lambda$
		\STATE \textbf{OPT\_SUBTREE}(0)
	\ENDMAIN
	\item[]
	\FUNCTION{\textbf{OPT\_SUBTREE}($t$)}
		\STATE initialize $\alpha_t$ and $\nu_t$
		\WHILE{$\mathscr{L}'$ increases}
			\FORALL{subcategory $c$ of $t$}
				\STATE \textbf{OPT\_SUBTREE}($c$)
			\ENDFOR
			\FORALL{document $d$ of $t$}
				\STATE \textbf{OPT\_DOCUMENT}($d$)
			\ENDFOR
			\STATE Update $\nu_t$ and $\alpha_t$
			\IF {t = 0} \STATE Update $\lambda$, $\eta$ and $\gamma$ \ENDIF
		\ENDWHILE
	\ENDFUNCTION
	\end{algorithmic}
\end{algorithm}

\subsection{Parallel Implementation}

Here we describe our scheme for parallelizing the recursive procedures in Algorithm 1.  In practice, we obtain a significant speedup from this parallel implementation of tiLDA.  This parallelization was necessary, for example, to obtain the results in Section~\ref{sec:experiment}.
 
One naive manner of parallelization would simply be to allocate the inference for different top-level categories to different threads of execution.  This approach, however, has two obvious limitations.  First, inference in different categories may require different amounts of time; if the goal is to minimize idle CPU cycles, then we must more intelligently distribute the overall workload across different threads. Second, the number of parallel threads at our disposal may greatly exceed the number of top-level categories.  (For example, the BlackHatWorld corpus has only three top-level categories.)  In this case, the naive approach to parallelization hardly makes the best use of available resources.  In the following, we describe a parallel implementation of tiLDA that overcomes both these limitations.

Our parallel implementation is based on two key ideas.  The first is to partition the algorithm into three types of tasks---START, DOCUMENT, and REPEAT---which we explain below.  The second is to maintain a queue of these tasks and create multiple threads that execute tasks from this queue.

A START task is associated with every internal node in the corpus hierarchy.  The task begins by initializing the node's parameters $\alpha_t$ and $\nu_t$.  After this initialization, the task then enqueues a new START task for each subcategory of the node and a DOCUMENT task for each document of the node.  In Algorithm 1, the START task corresponds to lines 5--10.

A DOCUMENT task is associated with each document in the corpus.  This task is responsible for optimizing the variational parameters $\nu_d$ and $\rho_{dn}$ for documents given their observed words and (currently inferred) parameters of their parents. In Algorithm 1, the DOCUMENT task corresponds to the procedure called in line~10.

A REPEAT task is issued at each internal node in the corpus whenever all the tasks for the node's children complete.  The REPEAT task is responsible for maximizing the lower bound on the log-likelihood $\mathscr{L}'$ with respect to the node's parameters.  We mark the node as {\it complete} if the lower bound does not improve over its value from the previous REPEAT task at the node.  Otherwise, we enqueue START and DOCUMENT tasks again for the node's children. The REPEAT task corresponds to executing lines 11--13 and then lines 6--10.

The overall algorithm begins with a START task at the root node and ends in a REPEAT task at the root node when the lower bound $\mathscr{L}'$ can no longer be improved.

\section{Experiments}
\label{sec:experiment}

In this section we evaluate tiLDA on several corpora and compare its results where possible to existing implementations of HDPs.  We followed more or less standard procedures in training.  The variational EM algorithms for tiLDA and HDPs were iterated until convergence in their log-likelihood bounds.  HDPs estimated by Gibbs sampling were trained for 5,000 iterations.
 
Since the log-likelihood of held-out data cannot be computed exactly in topic models, we use a method known as {\it document completion}~\citep{icml2009_139} to evaluate each model's predictive power.  First, for each trained model, we estimate a set of topics~$\beta$ and category topic proportions~$\theta_t$ on this set of topics. Then we split each document in the held-out set into two parts; on the first part, we estimate document topic proportions~$\theta_d$, and on the second part, we use these proportions to compute a per-word log-likelihood.  This approach permits a fair comparison of different (or differently trained) models.  

The topic proportions of held-out documents were computed as follows.  For variational inference, we simply estimated $\theta_d$ by ${\rm E}_q[\theta_d]$.  In the HDPs trained by Gibbs sampling, we sampled topic assignments $z_{dn}$ for each word in the first part of the document and computed $\theta_{dk}^s = \frac{ \alpha_{\pi(d)} \theta_{\pi(d)k} + N_{dk}^s }{ \alpha_{\pi(d)} + N_d}$, where $N_{dk}^s$ is the number of tokens assigned to $k$th topic in $s$th sample.  Finally, we averaged $\theta_{dk}^s$ over 2,000 samples after 500 iterations of burn-in.

\subsection{Comparison to HDPs}

HDPs have been evaluated on several ``flat" corpora, which in the manner of Figures~\ref{figure:freelancer_hierarchy}--\ref{figure:blackhatworld_hierarchy} we can visualize as {\it two-level} trees in which all documents are directly attached to a single root node.  In this section we compare the results from tiLDA and HDPs on three such corpora from the UCI Machine Learning Repository~\citep{Frank+Asuncion:2010}.  These corpora are: (1)~KOS---a collection of 3,430 blog articles with 467,714 tokens and a 6,906-term vocabulary; (2) Enron---a collection of 39,861 email messages with roughly 6 million tokens and a 28,102-term vocabulary; (3) NYTimes---a collection of 300K news articles with a 102,660-term vocabulary.   The full NYTimes corpus was too large for our experiments on (batch) HDPs so we extracted a subset of 80K articles with 26 million tokens.

On the KOS, Enron, and NYTimes corpora we compared tiLDA to two publicly available batch implementations\footnote{\url{http://www.stats.ox.ac.uk/~teh/software.html}}${^,}$\footnote{\url{http://www.cs.cmu.edu/~chongw/resource.html}} of HDPs, one based on Gibbs sampling~\citep{Teh:2006dy}, the other based on variational methods~\citep{Wang:2011ux}.  We denote the former by HDP-Gibbs and the latter by HDP-Variational.  For all algorithms we used the same hyperparameters ($\gamma\!=\!\alpha_0\!=\!1$) and the same symmetric Dirichlet prior on topics.  We initialized HDP-Gibbs with 100 topics, and we experimented with three settings of the truncation parameters ($K,T$) in HDP-Variational, where $K$ is the number of topics per corpus and $T$ is the number of topics per document.  These settings were $(K\!=\!150, T\!=\!15)$ as reported in previous work~\citep{Wang:2011ux} as well as $(K\!=\!300, T\!=\!50)$ and $(K\!=\!300, T\!=\!100)$.  For each corpus we only report the results from HDP-Variational for the {\it best} of these settings. In our experience, however, HDP-Variational was  sensitive to these settings, exhibiting the same or more variance than tiLDA over widely different choices for its {fixed} number of topics.

\begin{figure*}[t]
\vskip 0.1in
\begin{center}
\centerline{\includegraphics[width=7in]{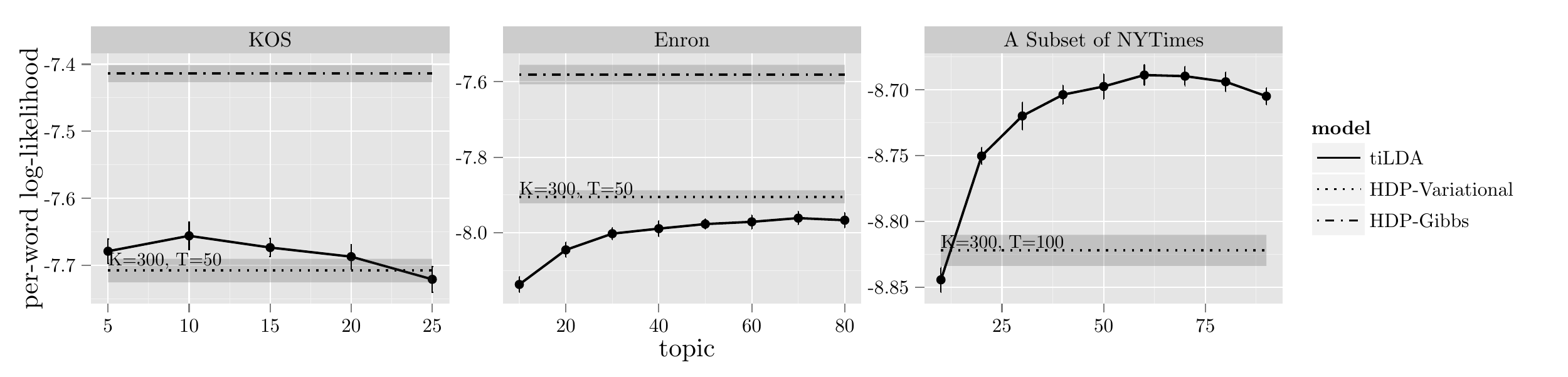}}
\caption{Predictive log-likelihood from two-level models of tiLDA and HDPs.  See text for details.}
\label{figure:2level}
\end{center}
\vskip -0.25in
\end{figure*}

Figure~\ref{figure:2level} summarizes our experimental results.  The error bars for tiLDA show the standard deviation in per-word log-likelihood over five different folds of each corpus.  (In each experiment, one fold was held out for testing while the other four were used for training.)  Also shown are the range of results on these folds for HDP-Gibbs and HDP-Variational.  On the smaller KOS and Enron corpora, we obtain our best results\footnote{It has been suggested that the careful selection of hyperparameters may reduce the gap between Gibbs sampling and variational methods in topic models~\citep{Asuncion:2009vm}; we did not explore that here.} with HDP-Gibbs; however, we emphasize that HDP-Gibbs was too slow to train even on our subset of the NYTimes corpus.  Comparing tiLDA and HDP-Variational, we find that the former does significantly better on the KOS and NYTimes corpora.  On Enron, the corpus which appears to contain the most topics, the order is reversed (but only provided that one explores the space of truncation parameters for HDP-Variational).  Though one cannot conclude too much from three corpora, these results certainly establish the viability and scalability of tiLDA.  We now turn to the sorts of applications for which tiLDA was explicitly conceived.


\subsection{Hierarchical Corpora}

In this section we demonstrate the benefits of tiLDA when it can exploit known hierarchical structure in corpora.  We experimented on three corpora with such structure.  These corpora are: (1) NIPS---a collection\footnote{\url{http://www.stats.ox.ac.uk/~teh/data.html}} of 1567 NIPS papers from 9 subject categories, with over 2 million tokens and a 13,649-term vocabulary; (2) Freelancer---a collection of 355,386 job postings by 6,920 advertisers, with over 16M tokens and a 27,600-term vocabulary, scraped from a large crowdsourcing site; (3) BlackHatWorld---a collection of 7,143 threads from a previously underground Internet forum, with roughly 1.4M tokens and a 7,056-term vocabulary.

The Freelancer corpus collects seven years of job postings from Freelancer.com, one of the largest crowdsourcing sites on the Internet.
We previously analyzed the Freelancer corpus using LDA~\citep{Kim_Topic_2011}, but this earlier work did not attempt to model the authorship of job postings as we do here.
The postings can be grouped by advertiser to form the three-level hierarchy shown in Figure~\ref{figure:freelancer_hierarchy}.  
In this hierarchy, tiLDA models the advertisers as second-level interior nodes and the job postings as third-level leaf nodes.  

The BlackHatWorld corpus collects over two years of postings from the ``BlackHatWorld'' Internet forum.
This data set was collected as part of a larger effort~\citep{Motoyama_Analysis_2011} to examine the social networks that develop in underground forums among distrustful parties.
The BlackHatWorld forum evolved to discuss abusive forms of Internet marketing, such as bulk emailing (spam).  The discussions are organized into the multi-level hierarchy shown in Figure~\ref{figure:blackhatworld_hierarchy}.
We treat the threads in these subforums as documents for topic modeling.  
(We do not consider individual posts within threads as documents because they are quite short.)
 
We preprocessed these two corpora in the same way, removing stopwords from a standard list~\citep{david_smart_2004}, discarding infrequent words that appeared in fewer than 6 documents,
and stemming the words that remain.
In both data sets, we also pruned ``barren" branches of the hierarchy: specifically, in the Freelancer corpus, we pruned advertisers with fewer than $20$ job postings, and in the BlackHatWorld corpus, we pruned subforums with fewer than $60$ threads.  

We evaluated three-level tiLDA models on the NIPS and Freelancer corpora (Figure~\ref{figure:freelancer_hierarchy}) and five-level tiLDA models on the BlackHatWorld corpus (Figure~\ref{figure:blackhatworld_hierarchy}).  For comparison we also evaluated two-level models of tiLDA that ignored the internal structure of these corpora.  We adopted the same settings as in the previous section except that we also learned the models' concentration parameters~$\alpha$.  Note that we do not have comparative results for multi-level HDPs on these corpora.  As of this time, we are unaware of Gibbs samplers for HDPs that would scale to corpora of this size and depth.  Likewise we are unaware of variational HDPs that have been implemented for general (multi-level) hierarchies.

\begin{figure*}[t]
\vskip 0.1in
\begin{center}
\centerline{\includegraphics[width=6.6in]{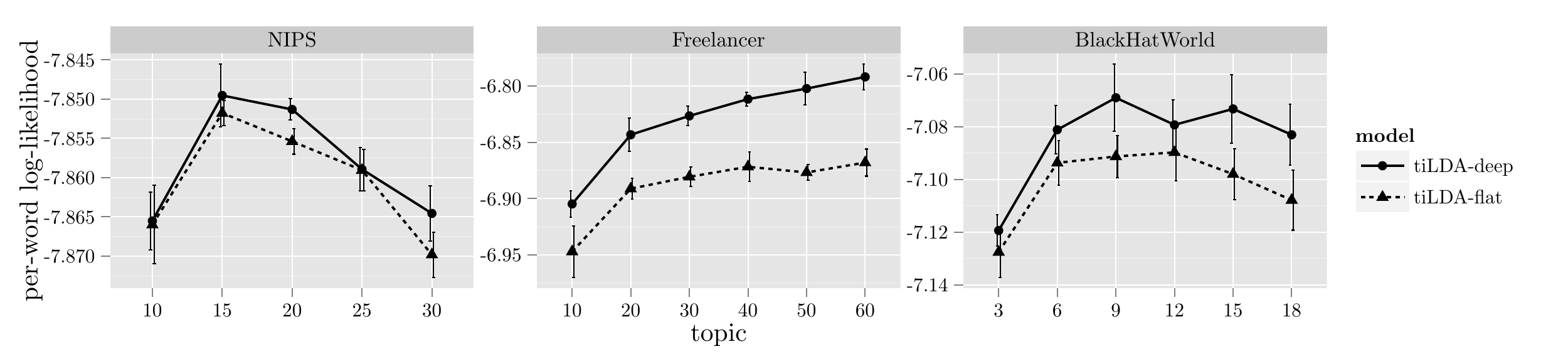}}
\caption{Predictive likelihood on the NIPS, Freelancer, and BlackHatWorld corpora from deep (multi-level) and flat (two-level) models of tiLDA, with varying numbers of topics.}
\label{figure:full_vs_2level}
\end{center}
\vskip -0.3in
\end{figure*} 

\subsubsection{Likelihood Evaluations}

Figure~\ref{figure:full_vs_2level} shows the results of these evaluations. The plot for the Freelancer corpus ({\it middle}) shows the average and standard deviation of the per-word log-likelihood over five folds. The plots for the NIPS and BlackHatWorld corpora ({\it left} and {\it right}) show the average and standard deviation over five runs, where each run averaged the test results over folds.  (We did this because the NIPS and BlackHatWorld corpora are much smaller, and the folds themselves exhibit large variance regardless of the settings.)

The results in Figure~\ref{figure:full_vs_2level}  paint a consistent picture over a wide range of choices for the number of topics,~$K$.  In every set of experiments, the deep tiLDA models of hierarchical corpora outperform the flat ones. Overall the results support the notion that deep tiLDA generalizes better for two reasons: first, because it can model different categories with different topic proportions, and second, because it shares information across different categories.  These abilities guard, respectively, against the challenges of underfitting and overfitting the data.

\subsubsection{Discovered Topics}

\begin{table}[tb] 
\begin{center} 
	\begin{tabular}{llll} 
		\multirow{ 2}{*}{\bf  ``OSN Linking''}	&\multirow{ 2}{*}{\bf  ``Ad Posting''}	&{\bf ``SEO Content}					&{\bf ``SEO Link} \\
									&							 &\multicolumn{1}{c}{\bf Generation''}	&\multicolumn{1}{c}{\bf Building''}  \\
		\hline\\
		facebook 		&ad			&articl		&post		\\ 
		fan 			&post		&keyword		&blog		\\ 
		friend 		&craigslist		&word		&forum	\\ 
		page 		&day			&topic		&comment		\\ 
		twitter 		&poster		&write		&link		\\ 
		Facebook	 	&section		&written		&site		\\ 
		account 		&citi			&origin		&Link Building		\\ 
		follow	 	&daili		&titl			&SEO		\\
	\end{tabular} 
\end{center} 
\caption{Four examples of the $K=60$ topics discovered by tiLDA on the Freelancer corpus; training time was 60 hours.
Shown are the eight most probable words for each topic.
Capitalized terms indicate project keywords.}
\label{table:topics_freelancer} 
\end{table}

We also examined the topics learned by the deep tiLDA models with the highest held-out likelihoods.  On the Freelancer corpus, which consists of job postings, these topics can be interpreted as different job types~\citep{Kim_Topic_2011}.  Table~\ref{table:topics_freelancer} shows four of the more pernicious job types on Freelancer.com identified by discovered topics.  The ``OSN (Online Social Network) Linking" topic describes jobs to generate friends and fans on sites such as Facebook and Twitter.  The ``Ad Posting" topic describes jobs to post classified ads on sites such as Craigslist.  Many other jobs are related to search engine optimization (SEO).  The ``SEO Content Generation" topic describes jobs to generate keyword-rich articles that drive traffic from search engines.  Likewise, the ``SEO Link Building" topic describes jobs to increase a Web site's Page\-Rank~\citep{Brin_Anatomy_1998} by adding links from blogs and forums.  

\begin{table}[tb] 
\begin{center} 
	\begin{tabular}{llll} 
		\multicolumn{1}{l}{\bf  ``Email}  		&\multicolumn{1}{l}{\bf  ``Google}  		&\multicolumn{1}{l}{\bf ``Affiliate}		&\multirow{2}{*}{\bf ``Blogging''}  	\\ 
		\multicolumn{1}{c}{\bf  Marketing''}		&\multicolumn{1}{c}{\bf  Adsense''}		&\multicolumn{1}{c}{\bf  Programs''}		&  	\\ 
		\hline\\
		email	&site			&DOLLAR	&blog	\\ 
		list		&traffic		&make		&forum	\\ 
		proxy	&googl		&money		&learn	\\ 
		ip		&adsens		&affili		&post	\\ 
		send		&domain		&market		&black	\\ 
		server	&ad			&product		&hat		\\ 
		address	&keyword		&month		&good	\\ 
		spam	&websit		&sell			&free	\\
	\end{tabular} 
\end{center} 
\caption{Four examples of the $K=9$ topics discovered by tiLDA on the BlackHatWorld corpus; training time was 30 minutes.
Shown are the eight most probable words for each topic. We replaced dollar amounts by the token DOLLAR.}
\label{table:topics_blackhatworld} 
\end{table}

On the BlackHatWorld corpus, the topics discovered by tiLDA relate to different types of Internet marketing.  Table~\ref{table:topics_blackhatworld} shows four particularly interpretable topics.
The ``Email Marketing" topic describes strategies for bulk emailing (spam).  The ``Google Adsense" topic describes ways for online publishers (e.g., bloggers) to earn money by displaying ads suggested by Google on their Web sites. The ``Affiliate Program" topic describes ways to earn commissions by marketing on behalf of other merchants. Finally, the ``Blogging" topic describes the use of blogs for Internet marketing.

\subsubsection{Analysis of Categories}

The multi-level tiLDA models can also be used to analyze hierarchical corpora in ways that go beyond the discovery of global topics.  Recall that each tiLDA model yields topic proportions $\theta_t$ and a concentration parameter $\alpha_t$ for each category of the corpus.  We can analyze these proportions and parameters for further insights into hierarchical corpora.  In general, they provide a wealth of information beyond what can be discerned from (say) ordinary LDA.

\begin{table}[tb]
\begin{center} 
	\begin{tabular}{ll|ll} 
		\multicolumn{1}{l}{\bf  Type}  		&\multicolumn{1}{l}{\bf  Ratio}  		&\multicolumn{1}{l}{\bf Type}		&\multicolumn{1}{l}{\bf Ratio} 	\\ 
		\hline
		SEO				&18.47\%		&OSN Linking		&2.12\% 	\\
		Affiliate Program	&3.21\%		&Bulk Emailing	&1.85\%		\\
		Captcha Solving	&2.68\%		&Account Creation	&1.42\% 	\\
		Ad Posting		&2.50\%		&Benign Jobs		&67.74\% \\
	\end{tabular} 
\end{center} 
\caption{Estimated ratio of number of buyers in job types on the Freelancer data set.}
\label{table:buyer_ratio}
\end{table}

Consider for example the Freelancer corpus.  In this corpus, the categories of tiLDA represent advertisers, and the topic proportions of these categories can be used to profile the types of jobs that advertisers are trying to crowdsource.  Summing these topic proportions over the corpus gives an estimate of the number of advertisers for each job type.  Table~\ref{table:buyer_ratio} shows the results of this estimate: it appears that nearly one-third of advertisers on Freelancer.com are commissioning abuse-related jobs, and of these jobs, the majority appear to involve some form of SEO.

We gain further insights by analyzing the concentration parameters of individual advertisers.  For example, the advertiser with the maximum concentration parameter ($\alpha_t=4065.00$) on Freelancer.com commissioned 34 projects, among which 32 have nearly the exact same description.  We also observe that advertisers with lower concentration parameters tend to be involved in a wider variety of projects.

On the BlackHatWorld corpus, the topic proportions and concentration parameters of categories generally reflect the titles of their associated subforums.  For example, the highest topic proportion (0.48) for ``Email Marketing" belongs to the subforum on `Email Marketing and Opt-In Lists,' and the highest topic proportion (0.59) for ``Blogging" belongs to the `Blogging' subforum.  The highest concentration parameter ($29.62$) belongs to the `Money, and Mo Money' subforum.  This is not surprising as this subforum itself has only four subforums as children, all of which are narrowly focused on specific revenue streams; see Figure~\ref{figure:blackhatworld_hierarchy}.

\section{Conclusion}
\label{sec:discuss}
In this paper we have explored a generalization of LDA for hierarchical corpora.  The parametric model that we introduce can also be viewed as a finite-dimensional HDP.  Our main technical contribution is Theorem~\ref{thm:bound}, which has many potential uses in graphical models with latent Dirichlet variables.  Our main empirical contribution is a parallel implementation that scales to very large corpora and deep hierarchies.  Our results on the Freelancer and BlackHatWorld corpora illustrate two real-world applications of our approach.

Unlike tiLDA, nonparametric topic models can infer the number of topics from data and grow this number as more data becomes available.  But this advantage of HDPs does not come without various complexities.  Variational inference in tiLDA does not require stick-breaking constructions or truncation schemes, and it generalizes easily to hierarchies of arbitrary depth.  For many applications, we believe that tiLDA provides a compelling alternative to the full generality of HDPs.   The approximations we have developed for tiLDA may also be useful for truncated versions of nonparametric models~\citep{Kurihara:2007vu}.  

We note one potential direction for future work.  In this paper, we have studied a batch framework for variational inference.  Online approaches, like those recently explored for LDA and HDPs~\citep{NIPS2012_0208, NIPS2012_1251, Hoffman:2013tz}, also seem worth exploring for tiLDA.  Such approaches may facilitate even larger and more diverse applications.

\appendix
\section{Proof of Theorem~\ref{thm:bound}}
\label{app:theorem}

The basic steps to prove 
Theorem~\ref{thm:bound} 
are contained in two lemmas.

\lemmaconcave*
\begin{proof}
We prove concavity by showing $f''(x)<0$ for all $x>0$.  Taking derivatives, we find:
\begin{equation}
	f''(x) = \Psi' (x)  - \frac{1}{x^2} - \frac{1}{x}, 
	\label{eq:secondDeriv}
\end{equation}
where $\Psi(x)$ denotes the digamma function and $\Psi'(x)$ its derivative.
A useful identity for this derivative~\citep{abramowitz_handbook_1964}
is the infinite series representation:  
\begin{equation}
	\Psi' (x) = \sum_{k=0}^{\infty}{ \frac{1}{(x+k)^2}  }.  \nonumber
\end{equation}
The lemma follows by substituting this series representation into eq.~(\ref{eq:secondDeriv}).  In particular, we have:
\begin{eqnarray*}
f''(x) 
  & = & - \frac{1}{x}\ +\ \sum_{k=1}^{\infty}{ \frac{1}{(x+k)^2}  } \\
  &< & - \frac{1}{x} + \frac{1}{x(x+1)} + \frac{1}{(x+1)(x+2)} +\ \cdots   \\ 
  &= & - \frac{1}{x} + \left[ \frac{1}{x}\!-\!\frac{1}{x+1} \right] + \left[ \frac{1}{x+1}\!-\! \frac{1}{x+2} \right] +\ \cdots   \\ 
		   &= & 0 
\end{eqnarray*}		  
This completes the proof, but we gain more intuition by plotting $f(x)$ as shown in Figure~\ref{fig:concave}.  Note that $\log \Gamma(x)$, which contains only the first term in~$f(x)$, is a {\it convex} function of $x$.  Thus it is the other terms in $f(x)$ that flip the sign of its second derivative.  Essentially, the concavity of $f(x)$ is established by adding $\log x$ at small $x$ and by subtracting $x\log x$ at large~$x$.
\end{proof}

\begin{figure}[tb] 
	\centering
	\includegraphics[width=3in]{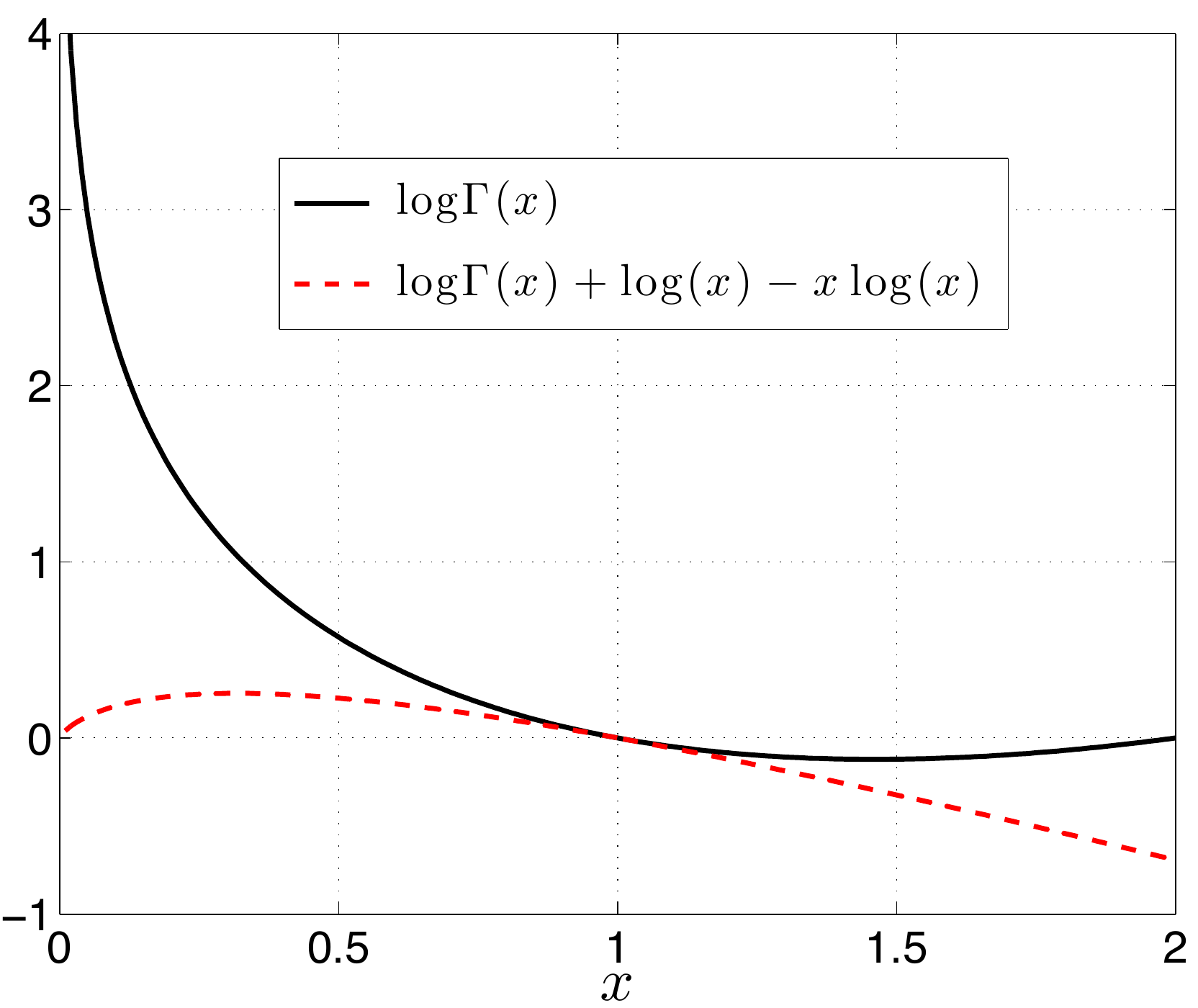}
\caption{Plots of the functions that appear in Lemma~\ref{lem:concave}.  The function $\log\Gamma(x)$ is convex, while the function $f(x) = \log \Gamma (x) + \log x - x \log x$ is concave.}
\label{fig:concave}
\end{figure}

\vspace{-2ex}
\lemmaloggammabound*
\begin{proof}
Let $f(x)$ denote a concave function on $x>0$.  From Jensen's inequality, we have that \mbox{${\rm E}[f(x)] \leq f({\rm E}[x])$}.  The result follows by setting \mbox{$f(x) = \log\Gamma(x) +  \log x - x\log x$} as in Lemma~\ref{lem:concave}.
\end{proof}

Note that a naive application of Jensen's inequality to the left hand side of eq.(\ref{eq:logGammaBound}) yields the {\it lower} bound ${\rm E}[\log\Gamma(x)] \geq \log\Gamma({\rm E}[x])$.  Thus it is the additional terms on the right hand side of eq.~(\ref{eq:logGammaBound}) that establish the {\it upper} bound.  The direction of this inequality is crucial in the context of variational inference, where the upper bound in eq.~(\ref{eq:logGammaBound}) is needed to maintain an overall lower bound on the log-likelihood.  Equipped with this lemma, we can now prove our main result.

\theorembound*
\begin{proof}
Let $\theta\sim {\rm Dirichlet}(\nu)$, and also let \mbox{$\alpha>0$}. Setting $x=\alpha\theta_i$ in eq.~(\ref{eq:logGammaBound}) gives:
\begin{equation}
{\rm E}[\log\Gamma(\alpha\theta_i)] \ \leq \   \log\Gamma(\alpha {\rm E}[\theta_i)]) + \log{\rm E}[\theta_i] - {\rm E}[\log \theta_i]  + \alpha {\rm E}[\theta_i \log\theta_i] - \alpha {\rm E}[\theta_i] \log {\rm E}[\theta_i]. 
\label{eq:bound2}
\end{equation}
All the expected values on the right hand side of this inequality can be computed analytically for Dirichlet random variables.  In particular, let $\nu_0=\sum_i \nu_i$.  Then:
\begin{eqnarray}
{\rm E}[\theta_i] & = & \frac{\nu_i}{\nu_0}, \nonumber \\
{\rm E}[\log\theta_i] & = & \Psi(\nu_i) - \Psi(\nu_0),\nonumber \\
{\rm E}[\theta_i\log\theta_i] & = & {\rm E}[\theta_i]\! \left(\!
  {\rm E}[\log\theta_i] + \frac{1}{\nu_i} - \frac{1}{\nu_0}\!\right)\!. \nonumber
\end{eqnarray}
The theorem follows from substituting these statistics into eq.~(\ref{eq:bound2}).
\end{proof}

How tight is the bound in Lemma~\ref{lem:logGammaBound}?  The question is important because we use this inequality in conjunction with the variational approximation in eq.~(\ref{eq:variational}) to generate a lower bound on the log-likelihood.  Here we make two useful observations.

\begin{figure}[tb] 
	\centering
	  \fbox{\includegraphics[width=3in]{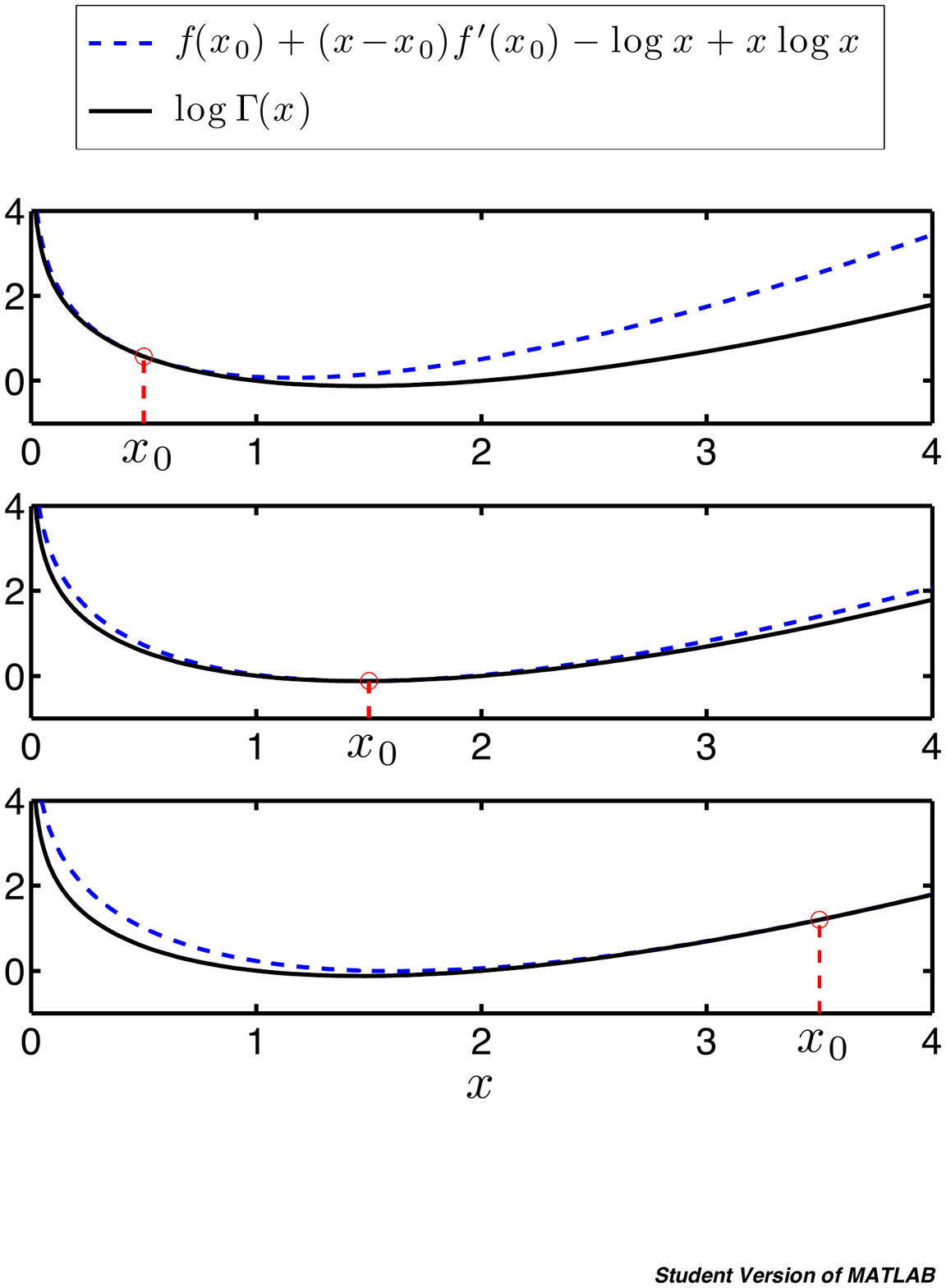}}
\caption{Plots of $\log\Gamma(x)$ and its upper bound in eq.~(\ref{eq:logGammaBound2}) for different values of $x_0$.  Note the tightness of the bounds in the vicinity of $x_0$.}
\label{fig:logGammaBound}
\end{figure}

First, we note that the bound in Lemma~\ref{lem:logGammaBound} is exquisitely tuned to the shape of the function $\log\Gamma(x)$ and the location of the expected value ${\rm E}[x]$.  To see this, we provide an alternate derivation of the result in eq.~(\ref{eq:logGammaBound}).  We begin by appealing to the concavity of $f(x)$, established in Lemma~\ref{lem:concave}; from this we
obtain the upper bound
\begin{equation}
f(x)\ \leq\ f(x_0) + f'(x_0)(x-x_0), \nonumber
\end{equation}
which holds for all values $x_0\!>\!0$.  Now we recall the definition of $f(x)$ in Lemma~\ref{lem:concave} to obtain an upper bound on $\log\Gamma(x)$.  Specifically we have:
\begin{eqnarray}
\log\Gamma(x)\!\!\!\!
  & =\!\!\!\!& f(x) - \log x + x\log x,\mbox{\hspace{23ex}} \nonumber \\ 
  & \leq\!\!\!\! & f(x_0)\!\, +\!\, f'\!(x_0\!\,)(x\!-\!x_0\!\,)\!\, -\!\, \log\!\,x\!\, +\!\, x\log\!\, x. 
  \label{eq:logGammaBound2}
\end{eqnarray}
Figure~\ref{fig:logGammaBound} illustrates this upper bound on $\log\Gamma(x)$ for different values of $x_0$; note especially its tightness in the vicinity of $x_0$.  The upper bound on ${\rm E}[\log\Gamma(x)]$ in eq.~(\ref{eq:logGammaBound}) is based on choosing the best approximation from this family of upper bounds; it is easy to show that this occurs at $x_0 = {\rm E}[x]$.  Thus we obtain the bound in Lemma~\ref{lem:logGammaBound} by taking expectations of both sides of eq.~(\ref{eq:logGammaBound2}) and setting $x_0 = {\rm E}[x]$.   

Second, we note that the upper bound in eq.~(\ref{eq:logGammaBound}) reduces to an equality in the limit of vanishing variance.  In particular, this is the limit in which ${\rm E}[\log x] \rightarrow \log {\rm E}[x]$ and also ${\rm E}[x \log x] \rightarrow {\rm E}[x] \log {\rm E}[x]$.  In this limit, the last four terms on the right hand side of eq.~(\ref{eq:logGammaBound}) vanish, and we recover the result ${\rm E}[\log \Gamma(x)] = \log\Gamma({\rm E}[x])$.  In general, we expect factorized approximations such as eq.~(\ref{eq:variational}) to work well in the regime where the true posterior is peaked around its mean value.  In this regime, we also expect the bound in eq.~(\ref{eq:logGammaBound}) to be tight.  Put another way, if it is sufficiently accurate to proceed with the factorized approximation in eq.~(\ref{eq:variational}), then we do not expect to incur much additional loss from the inequality in Lemma~\ref{lem:logGammaBound}.

\section{Inference and Parameter Estimation}
In this appendix, we derive detailed procedures for inference and parameter estimation.
We first derive the lower bound $\mathscr{L}'$ on the log-likelihood for tiLDA, then show how to maximize it with 
respect to both the variational parameters and the model parameters.

\subsection{Variational Lower Bound}
\label{app:var_obj}

The variational lower bound $\mathscr{L}$ in eq.~(\ref{eqn:elbo}) consists of two terms: the first depends on the evidence, while the second $H(q)$ computes the entropy of the variational distribution.  We expand the first evidence-dependent term as:
\begin{eqnarray}
\lefteqn{{\rm E}_q \left[ \log p(\theta, z, w, \beta | \gamma, \alpha, \eta ) \right]}\mbox{\hspace{3ex}} \nonumber \\ \nonumber \\
  & = & {\rm E}_q \left[ \log p(\theta_0 | \gamma) \right]\ +\ \sum_{t > 0}  { {\rm E}_q \left[ \log p( \theta_t | \alpha_{\pi(t)} \theta_{\pi(t)}   ) \right] }\  +\ \sum_{d}  { {\rm E}_q \left[ \log p( \theta_d | \alpha_{\pi(d)} \theta_{\pi(d)}   ) \right] } \nonumber \\
	& & +\ \sum_d \sum_{n=1}^{N_d} \Big[ {\rm E}_q \left[ \log p( z_{dn} | \theta_d) \right] + {\rm E}_q \left[ \log p(w_{dn} | z_{dn}, \beta) \right]\! \Big]\   +\ \sum_{k=1}^K {\rm E}_q \left[ \log p(\beta_k | \eta) \right]\!.\
\label{eqn:evidenceTerm}
\end{eqnarray}
Here again we use $\pi(t)$ and $\pi(d)$ to denote the parent categories of category $t$ and document $d$, respectively, in the tree.  In the above equation, only the second and third terms are new to tiLDA; the other terms are familiar from ordinary LDA, so we do not consider them further.  The two new terms have the same form, both involving expectations of Dirichlet random variables; for concreteness we focus on the term ${\rm E}_q[ \log p(\theta_t | \alpha_{\pi(t)} \theta_{\pi(t)} )]$ involving internal categories of the tree.  Recall the form of the Dirichlet distribution:
\begin{equation}
\log p(\theta_t | \alpha_{\pi(t)} \theta_{\pi(t)} )\ =\
  \log\Gamma\!\left(\alpha_{\pi(t)}\right)
  - \sum_{i=1}^K \log\Gamma\!\left(\alpha_{\pi(t)}\theta_{\pi(t)i}\right)
  + \sum_{i=1}^{K} \left[\alpha_{\pi(t)}\theta_{\pi(t)i}-1\right] \log \theta_{ti}.
\label{eq:logDirichlet}
\end{equation}
Note that the first term on the right hand side of eq.~(\ref{eq:logDirichlet}) is a constant, while the expected value of the third term is computed easily from the independence assumptions of the variational distribution.  
In particular, we have:
\begin{eqnarray*}
{\rm E}_q[ \theta_{\pi(t)i} \log \theta_{ti}] & = & {\rm E}_q[ \theta_{\pi(t)i}]\, {\rm E}_q[ \log \theta_{ti}], \\
{\rm E}_q[ \theta_{\pi(t)i}] & = & {\nu_{\pi(t)i}} / {\nu_{\pi(t)0}}, \\
{\rm E}_q[ \log \theta_{ti}] & = & \Psi(\nu_{ti}) - \Psi(\nu_{t0}), \\
\end{eqnarray*}
where $\nu_{t0} = \sum_{i=1}^K {\nu_{ti}}$.  We cannot exactly compute the expected value of the second term on the right hand side of eq.~(\ref{eq:logDirichlet}), which involves the intractable average ${\rm E}_q[ \log \Gamma( \alpha_{\pi(t)} \theta_{\pi(t)i})]$.   However, we can bound this expected value using Theorem~\ref{thm:bound}.  Combining this result with the above, we obtain the overall lower bound:
\begin{eqnarray}
\lefteqn{{\rm E}_q\bigg[ \log p(\theta_t | \alpha_{\pi(t)} \theta_{\pi(t)} )\bigg]}\nonumber\\ \nonumber\\
   & \ge & \log \Gamma \left( \alpha_{\pi(t)} \right) - \frac{ \alpha_{\pi(t)} }{ \nu_{\pi(t)0} } ( K\!-\! 1) - ( \alpha_{\pi(t)}\! -\! K ) \bigg(\! \log \nu_{\pi(t)0} - \Psi\!\left( \nu_{\pi(t)0} \right) +  \Psi\!\left(\nu_{t0}\right)\!\bigg)\nonumber \\
	 & & - \sum_{i=1}^{K}\! \left[ \log \Gamma\! \left( \alpha_{\pi(t)} \frac{ \nu_{\pi(t)i} }{ \nu_{\pi(t)0} } \right) + \left( 1\! -\! \alpha_{\pi(t)} \frac{ \nu_{\pi(t)i} }{ \nu_{\pi(t)0} } \right)\! \bigg(\! \log \nu_{\pi(t)i} - \Psi( \nu_{\pi(t)i} ) + \Psi(\nu_{ti})\! \bigg)\right] .
\end{eqnarray}
This result can be used to provide a lower bound on the second and third terms in eq.~(\ref{eqn:evidenceTerm}).  Finally, combining this lower bound with familiar terms from variational inference in LDA~\citep{blei_latent_2003}, we obtain a new lower bound $\mathscr{L}'$ on the log-likelihood for tiLDA.  The new bound is slightly looser than the standard one in eq.~(\ref{eqn:elbo}), but it remains rigorous and straightforward to compute in terms of the variational parameters.

\begin{algorithm}[tb]
\caption{Inference algorithm for a document.} 
	\label{alg:doc}
	\begin{algorithmic}[1]
		\FUNCTION{\textbf{OPT\_DOCUMENT}($d$)}
			\STATE initialize $\rho_{dni} = 1/K$ for all $n$ and $i$
			\STATE initialize $\nu_{di}  = \alpha_{\pi(d)} \frac{ \nu_{\pi(d)i} }{ \nu_{\pi(d)0} } + N_d / K $ for all $i$
			\WHILE{$\mathscr{L}'$ increases}
				\FOR{$n = 1$ to $N_d$}
					\FOR{$i = 1$ to $K$}
						\STATE $\rho_{dni} =  \exp \big\{ \Psi(\nu_{di}) +  \Psi( \lambda_{iw_{dn}} ) - \Psi( \lambda_{i0}  ) \big\}$
					\ENDFOR
					\STATE normalize $\rho_{dni}$ so that $\sum_{i=1}^K{\rho_{dni}} = 1$
				\ENDFOR
				\STATE $\nu_{di} =  \alpha_{\pi(d)} \frac{ \nu_{\pi(d)i} }{ \nu_{\pi(d)0} } + \sum_{n=1}^{N_d}{ \rho_{dni} }$	for all $i$
			\ENDWHILE
		\ENDFUNCTION
	\end{algorithmic}
\end{algorithm}

\subsection{Inference on Documents}
\label{app:max_doc_params}

There are variational parameters $\nu_d$ for the topic proportions $\theta_d$ of each document, as well as variational parameters $\rho_{dn}$ for the topics $z_{dn}$ of each word $w_{dn}$.  Algorithm 2 shows how we update these variational parameters.  Our derivation of these updates follows closely that of previous work on LDA~\citep[Appendix~A.3]{blei_latent_2003}.

We first maximize our lower bound $\mathscr{L}'$ on the log-likelihood, computed in Appendix~\ref{app:var_obj}, with respect to $\rho_{dn}$.  To this end, it is helpful to isolate just those terms in $\mathscr{L}'$ that involve the parameters~$\rho_{dn}$.  These terms are given by:
\begin{equation}
\mathscr{L}'_{[\rho_{dn}]} =  \sum_{i=1}^{K}{\rho_{dni} \bigg[ \Psi(\nu_{di}) - \Psi( \nu_{d0} ) + \Psi( \lambda_{iw_{dn}} ) - \Psi( \lambda_{i0}  ) - \log \rho_{dni} \bigg] }.
\label{eqn:terms_rho}
\end{equation}
Here and in what follows, we use the notation $\mathscr{L}'_{[\rho_{dn}]}$ to indicate the subset of terms in~$\mathscr{L}'$ that have some dependence on the subscripted parameters.
Optimizing eq.~(\ref{eqn:terms_rho}) subject to the multinomial constraint $\sum_{i=1}^K{\rho_{dni}} = 1$,
we obtain the variational update:
\begin{equation}
\rho_{dni} \propto  \exp \big\{ \Psi(\nu_{di}) +  \Psi( \lambda_{iw_{dn}} ) - \Psi( \lambda_{i0}  )  \big\}. \nonumber
\end{equation}
Similarly, we can isolate the terms in $\mathscr{L}'$ that depend on the variational parameters $\nu_d$.  These terms are given by:
\begin{equation}
\mathscr{L}'_{[\nu_{d}]} =  \sum_{i=1}^{K}{ \bigg[ \Psi( \nu_{di} ) - \Psi( \nu_{d0} )  \bigg]  \left[ \alpha_{\pi(d)} \frac{ \nu_{\pi(d)i} }{ \nu_{\pi(d)0} } + \sum_{n=1}^{N_d}{ \rho_{dni} - \nu_{di} } \right]  }\  -\ \log \Gamma( \nu_{d0} )\ +\  \sum_{i=1}^{K}{ \log \Gamma(\nu_{di})  }. \nonumber
\end{equation}
We obtain the variational update for the topic proportion parameters by optimizing $\mathscr{L}'_{[\nu_{d}]}$ with respect to $\nu_{dj}$.  Taking the derivative, we find:
\begin{eqnarray*}
\frac{\partial\mathscr{L}'}{\partial\nu_{dj}} 
  & =  & \sum_{i=1}^{K}{ \left\{\frac{\partial}{\partial\nu_{dj}}\bigg[ \Psi( \nu_{di} ) - \Psi( \nu_{d0} )  \bigg]\right\}\!\!  \left[ \alpha_{\pi(d)} \frac{ \nu_{\pi(d)i} }{ \nu_{\pi(d)0} } + \sum_{n=1}^{N_d}{ \rho_{dni} - \nu_{di} } \right]  } \\ \\
  & & -\ \bigg[ \Psi( \nu_{dj} ) - \Psi( \nu_{d0} )  \bigg]\ +\ \bigg[ \Psi( \nu_{dj} ) - \Psi( \nu_{d0} )  \bigg].
\end{eqnarray*}
The terms in the second line of this expression cancel each other out.  The variational update is obtained by requiring the remaining term in the first line to vanish as well.  The resulting update is given by:
\begin{equation}
\nu_{di}\ =\  \alpha_{\pi(d)} \frac{ \nu_{\pi(d)i} }{ \nu_{\pi(d)0} }\ +\ \sum_{n=1}^{N_d}{ \rho_{dni} }. \nonumber
\end{equation}
The form of this update has a simple intuition---namely, that the prior over each document's topic proportions~$\theta_d$ is inherited from its parent category~$\pi(d)$ with a concentration parameter of $\alpha_{\pi(d)}$ and a base measure of $\frac{ \nu_{\pi(d)}}{\nu_{\pi(d)0}}$.
\subsection{Inference on Categories}
\label{app:max_cat_params}

There are variational parameters $\nu_t$ for the topic proportions of each category.  Once again we vary these parameters to maximize the lower bound~$\mathscr{L}'$ on the log-likelihood.  For mathematical convenience, we decompose $\nu_t$ into a concentration parameter $\tau_t$ and a base measure $\kappa_t$, and we estimate these quantities separately~\citep{Minka:pjtYdr1C}.

To maximize $\mathscr{L}'$ with respect to $\kappa_t$, we use a Newton's method with equality constraints~\citep{boyd_convex_opt_2004}.  By virtue of the decomposition, we will see that each Newton update takes only linear time in $K$, the number of topics.  We begin by extracting the terms that involve $\kappa_t$ from the lower bound~$\mathscr{L}'$ on the log-likelihood:
\begin{align}
\mathscr{L}'_{[\kappa_t]} = &\sum_{i=1}^{K}{ \Psi(\tau_t \kappa_{ti})\! \left[ \delta_{t0} \frac{\gamma}{K} + (1 - \delta_{t0}) \alpha_{\pi(t)} \kappa_{\pi(t)i} + |C_t| (1 - \alpha_t \kappa_{ti}) - \tau_t \kappa_{ti} \right] } \nonumber \\
	&- |C_t| \sum_{i=1}^{K}{ \Big[ \log \Gamma( \alpha_t \kappa_{ti}) + (1 - \alpha_t \kappa_{ti}) \log \kappa_{ti} \Big] } + \sum_{i=1}^{K}{ \log \Gamma(\tau_t \kappa_{ti}) } \nonumber  \\
	&+ \alpha_t \sum_{i=1}^{K}{ \kappa_{ti} \left[  \sum_{d \in C_t}{ \Psi(\nu_{di}) } + \sum_{c \in C_t}{ \Psi(\tau_c \kappa_{ci}) }  \right]  },
\label{eqn:elbo_kappa}
\end{align}
where the discrete delta function $\delta_{t0}$ equals $1$ if $t$ equals $0$ (meaning $t$ represents the root node) and equals zero otherwise.  Here again we have used $C_t$ to denote the set of indexes for the subcategories and documents of category $t$.  From eq.~(\ref{eqn:elbo_kappa}), it can be verified that the Hessian of $\mathscr{L}'$ with respect to $\kappa_t$ is diagonal.  Let $g$ and $h$ be $K$-dimensional vectors that denote, respectively, the gradient with respect to $\kappa_t$ and the diagonal entries of the Hessian matrix.  Then the constrained Newton step $\Delta \kappa_t$ is given by:
\begin{equation}
	\begin{bmatrix}
		\textrm{diag}(h) & \mathbf{1} \\
		\mathbf{1}^T & 0
	\end{bmatrix}
	\begin{bmatrix}
		\Delta \kappa_t \\
		u
	\end{bmatrix}
	=
	\begin{bmatrix}
		-g \\
		0
	\end{bmatrix},
\label{eqn:matrix_kappa}
\end{equation}
where $u$ denotes the dual variable for the sum-to-one constraint.
Solving the linear set of equations in eq.~(\ref{eqn:matrix_kappa}), we obtain the update rule:
\begin{equation}
	\Delta \kappa_t = \left\{ \frac{ \sum_{i=1}^K{\frac{g_i}{h_i}} } { \sum_{i=1}^K{ \frac{1}{h_i}} } \right\}
	\begin{bmatrix}
	\frac{1}{h_1} \\
	\vdots \\
	\frac{1}{h_K} \\
	\end{bmatrix}
	-
	\begin{bmatrix}
	\frac{g_1}{h_1} \\
	\vdots \\
	\frac{g_K}{h_K} \\
	\end{bmatrix}. \nonumber
\end{equation}
It is straightforward to verify that $\sum_i \Delta\kappa_{ti} = 0$, and hence, by construction, this update preserves the normalization $\sum_i \kappa_{ti}=1$.
The step size is found by a backtracking line search~\citep{boyd_convex_opt_2004}.
We also guarantee that the components of the base measure remain nonnegative.

To maximize $\mathscr{L}'$ with respect to $\tau_t$, we begin again by extracting those terms in the lower bound that depend on $\tau_t$.  These terms are given by:
\begin{eqnarray}
\mathscr{L}'_{[\tau_t]} 
  & = & \sum_{i=1}^K{ \left\{\Big[ \Psi(\tau_t \kappa_{ti}) - \Psi( \tau_t ) \Big] \left[ \delta_{t0} \frac{\gamma}{K} + (1-\delta_{t0}) \alpha_{\pi(t)} \kappa_{\pi(t)i} + |C_t| (1 - \alpha_t \kappa_{ti}) - \tau_t \kappa_{ti} \right]\right\}} \nonumber \\
 & &  -\ \log \Gamma( \tau_t )\ -\ |C_t| \alpha_t \tau_t^{-1} (K - 1)\ +\ \sum_{i=1}^{K}\,{\log \Gamma( \tau_t \kappa_{ti})}.
\label{eqn:elbo_tau}
\end{eqnarray}
It is straightforward to compute the first and second derivatives of eq.~(\ref{eqn:elbo_tau}) with respect to~$\tau_t$.
We alternate maximizing $\mathscr{L}'$ with respect to $\kappa_t$ and $\tau_t$ until convergence.

\subsection{Inference on Topics}
\label{app:max_topic_params}

There is, finally, a variational parameter $\lambda_k$ for each topic's distribution over words $\beta_k$.
We update the parameters $\lambda_k$ (see Algorithm 1) after each sweep of the recursion from the documents to the root node.  To maximize $\mathscr{L}'$ with respect to $\lambda_k$, we first isolate the terms in the bound that involve $\lambda_k$.  These terms are given by:
\begin{equation}
\mathscr{L}'_{[\lambda_k]} =  \sum_{v=1}^V \Big[  \Psi(\lambda_{kv}) - \Psi( \lambda_{k0} ) \Big] \Big[  \frac{ \eta }{ V } + \sum_d{ \sum_{n=1}^{N_d} \rho_{dnk} w_{dn}^{v}} - \lambda_{kv} \Big]\  -\ \log \Gamma( \lambda_{k0} )\ +\ \sum_{v=1}^V \log \Gamma( \lambda_{kv} ),  \nonumber
\end{equation}
where $w_{dn}^{v}\in\{0,1\}$ is an indicator variable, equal to one only when $w_{dn}$ equals $v$, and where as shorthand we use $\lambda_{k0} = \sum_v \lambda_{kv}$.  The update for $\lambda_k$ is computed by taking its derivative and setting it equal to zero.  Not surprisingly, this update takes a similar form as the update for $\nu_d$ in appendix~\ref{app:max_doc_params}.  Specifically, we have:
\begin{equation}
\lambda_{kv} = \frac{ \eta }{ V }\ +\ \sum_d{ \sum_{n=1}^{N_d}{ \rho_{dnk} w_{dn}^{v} } }.  \nonumber
\end{equation}

\subsection{Parameter Estimation}
\label{app:estimate_params}

In the M-step of the variational EM algorithm, we must update the model parameters~$\gamma$ and~$\eta$ (for the corpus as a whole) and the concentration parameter $\alpha_t$ for each category~$t$.
Again we estimate these parameters by maximizing the lower bound~$\mathscr{L}'$ on the log-likelihood.
In this case the relevant terms in $\mathscr{L}'$ are given by:
\begin{eqnarray}
\mathscr{L}'_{[\gamma]} 
 & = & \ \log\Gamma( \gamma )\ -\ K \log\Gamma\!\left( \frac{ \gamma }{ K } \right)\ +\ \left(\frac{ \gamma }{ K } -1\right) \sum_{i=1}^K{ \Big[ \Psi ( \nu_{0i} ) - \Psi ( \nu_{00} ) \Big] } \nonumber \\
\mathscr{L}'_{[\eta]} 
  & =& \ K \log\Gamma( \eta )\ -\ K V \log\Gamma\!\left( \frac{ \eta }{ V }\right)\ +\ \left(\frac{ \eta }{ V } -1\right) \sum_{i=1}^K{ \sum_{v=1}^V{ \Big[ \Psi ( \lambda_{iv} ) - \Psi ( \lambda_{i0} ) \Big] } } \nonumber \\
\mathscr{L}'_{[\alpha_t]} 
  & =& \ |C_t| \left\{ \log \Gamma( \alpha_t ) + \alpha_t \Psi(\nu_{t0}) - \frac{\alpha_t}{\nu_{t0}} (K-1) -  \sum_{i=1}^{K}{ \left[ \log \Gamma\!\left(\frac{ \alpha_t \nu_{ti} }{ \nu_{t0} }\right) - \frac{ \alpha_t \nu_{ti} }{ \nu_{t0} } \left(\log \frac{\nu_{ti}}{\nu_{t0}} - \Psi( \nu_{ti} ) \right)  \right] }  \right\} \nonumber \\
	& & +\ \alpha_t \sum_{i=1}^{K}{ \frac{ \nu_{ti} }{ \nu_{t0} } \left\{  \sum_{c \in C_t}{ \Big[ \Psi(\nu_{ci}) - \Psi( \nu_{c0} ) \Big] } + \sum_{d \in C_t}{ \Big[ \Psi(\nu_{di}) - \Psi( \nu_{d0} ) \Big] } \right\} }. \nonumber
\end{eqnarray}
For the M-step of the variational EM algorithm, we update the model parameters by computing first and second derivatives of these terms in~$\mathscr{L}'$ and using Newton's method.

\vskip 0.2in
\bibliography{paper}

\begin{thebibliography}{30}
\providecommand{\natexlab}[1]{#1}
\providecommand{\url}[1]{\texttt{#1}}
\expandafter\ifx\csname urlstyle\endcsname\relax
  \providecommand{\doi}[1]{doi: #1}\else
  \providecommand{\doi}{doi: \begingroup \urlstyle{rm}\Url}\fi

\bibitem[Abramowitz and Stegun(1964)]{abramowitz_handbook_1964}
Milton Abramowitz and Irene~A. Stegun.
\newblock \emph{{Handbook of Mathematical Functions with Formulas, Graphs, and
  Mathematical Tables}}.
\newblock Dover, New York, 1964.

\bibitem[Adams et~al.(2010)Adams, Ghahramani, and Jordan]{Adams_Tree_2010}
Ryan Adams, Zoubin Ghahramani, and Michael Jordan.
\newblock {Tree-structured stick breaking for hierarchical data}.
\newblock In \emph{Advances in Neural Information Processing Systems 23}, pages
  19--27. 2010.

\bibitem[Asuncion et~al.(2009)Asuncion, Welling, Smyth, and
  Teh]{Asuncion:2009vm}
Arthur Asuncion, Max Welling, Padhraic Smyth, and Yee~Whye Teh.
\newblock {On smoothing and inference for topic models}.
\newblock In \emph{Proceedings of the Twenty-Sixth Conference on Uncertainty in
  Artificial Intelligence (UAI)}, 2009.

\bibitem[Blei(2014)]{Blei:2014cp}
David~M Blei.
\newblock {Build, Compute, Critique, Repeat: Data Analysis with Latent Variable
  Models}.
\newblock \emph{Annual Review of Statistics and Its Application}, 1\penalty0
  (1):\penalty0 203--232, January 2014.

\bibitem[Blei and Lafferty(2009)]{blei_topic_2009}
David~M Blei and John Lafferty.
\newblock {Topic models}.
\newblock In \emph{Text Mining: Theory and Applications}. Taylor and Francis,
  London, {UK}, 2009.

\bibitem[Blei et~al.(2003)Blei, Ng, and Jordan]{blei_latent_2003}
David~M Blei, Andrew~Y Ng, and Michael~I Jordan.
\newblock {Latent Dirichlet allocation}.
\newblock \emph{The Journal of Machine Learning Research}, 3:\penalty0
  993–--1022, March 2003.

\bibitem[Blei et~al.(2010)Blei, Griffiths, and Jordan]{Blei_Nested_2010}
David~M. Blei, Thomas~L. Griffiths, and Michael~I. Jordan.
\newblock The nested chinese restaurant process and bayesian nonparametric
  inference of topic hierarchies.
\newblock \emph{J. ACM}, 57\penalty0 (2):\penalty0 7:1--7:30, February 2010.

\bibitem[Boyd and Vandenberghe(2004)]{boyd_convex_opt_2004}
Stephen Boyd and Lieven Vandenberghe.
\newblock \emph{Convex Optimization}.
\newblock Cambridge University Press, March 2004.

\bibitem[Brin and Page(1998)]{Brin_Anatomy_1998}
Sergey Brin and Lawrence Page.
\newblock The anatomy of a large-scale hypertextual web search engine.
\newblock In \emph{Proceedings of the Seventh International Conference on the
  World Wide Web}, pages 107--117, 1998.

\bibitem[Bryant and Sudderth(2012)]{NIPS2012_1251}
Michael Bryant and Erik Sudderth.
\newblock {Truly nonparametric online variational inference for hierarchical
  Dirichlet processes}.
\newblock In \emph{Advances in Neural Information Processing Systems 25}, pages
  2708--2716. 2012.

\bibitem[Du et~al.(2010)Du, Buntine, and Jin]{Du:2010ff}
Lan Du, Wray Buntine, and Huidong Jin.
\newblock {A segmented topic model based on the two-parameter Poisson-Dirichlet
  process}.
\newblock \emph{Machine Learning}, 81:\penalty0 5--19, 2010.

\bibitem[Frank and Asuncion(2010)]{Frank+Asuncion:2010}
A.~Frank and A.~Asuncion.
\newblock {UCI} machine learning repository, 2010.

\bibitem[Hoffman et~al.(2013)Hoffman, Blei, Wang, and Paisley]{Hoffman:2013tz}
Matthew~D Hoffman, David~M Blei, Chong Wang, and John Paisley.
\newblock {Stochastic variational inference}.
\newblock \emph{The Journal of Machine Learning Research}, 14\penalty0
  (1):\penalty0 1303--1347, 2013.

\bibitem[Jordan et~al.(1999)Jordan, Ghahramani, Jaakkola, and
  Saul]{jordan_introduction_1999}
Michael~I Jordan, Zoubin Ghahramani, Tommi~S Jaakkola, and Lawrence~K Saul.
\newblock {An Introduction to variational methods for graphical models}.
\newblock \emph{Machine Learning}, 37\penalty0 (2):\penalty0 183–--233,
  November 1999.

\bibitem[Kim et~al.(2011)Kim, Motoyama, Voelker, and Saul]{Kim_Topic_2011}
Do-kyum Kim, Marti Motoyama, Geoffrey~M. Voelker, and Lawrence~K. Saul.
\newblock {Topic modeling of freelance job postings to monitor Web service
  abuse}.
\newblock In \emph{Proceedings of the 4th ACM Workshop on Security and
  Artificial Intelligence (AISec-11)}, pages 11--20, 2011.

\bibitem[Kim et~al.(2013)Kim, Voelker, and Saul]{Kim_tiLDA_2013}
Do-kyum Kim, Geoffrey Voelker, and Lawrence~K. Saul.
\newblock A variational approximation for topic modeling of hierarchical
  corpora.
\newblock In \emph{Proceedings of the 30th International Conference on Machine
  Learning (ICML-13)}, 2013.

\bibitem[Kurihara et~al.(2007)Kurihara, Welling, and Teh]{Kurihara:2007vu}
Kenichi Kurihara, Max Welling, and Yee~Whye Teh.
\newblock {Collapsed variational Dirichlet process mixture models}.
\newblock In \emph{Proceedings of the Twentieth International Joint Conference
  on Artificial Intelligence}, 2007.

\bibitem[Lewis et~al.(2004)Lewis, Yang, Rose, and Li]{david_smart_2004}
David~D. Lewis, Yiming Yang, Tony~G. Rose, and Fan Li.
\newblock {SMART stopword list}.
\newblock
  \url{http://jmlr.csail.mit.edu/papers/volume5/lewis04a/a11-smart-stop-list/english.stop},
  April 2004.

\bibitem[Li and McCallum(2006)]{Li_Pachinko_2006}
Wei Li and Andrew McCallum.
\newblock Pachinko allocation: Dag-structured mixture models of topic
  correlations.
\newblock In \emph{Proceedings of the 23rd International Conference on Machine
  Learning (ICML-06)}, pages 577--584, 2006.

\bibitem[Liang et~al.(2007)Liang, Petrov, Jordan, and Klein]{liang-EtAl:2007}
Percy Liang, Slav Petrov, Michael Jordan, and Dan Klein.
\newblock {The Infinite PCFG Using Hierarchical Dirichlet Processes}.
\newblock In \emph{Proceedings of the 2007 Joint Conference on Empirical
  Methods in Natural Language Processing and Computational Natural Language
  Learning (EMNLP-CoNLL)}, 2007.

\bibitem[Minka(2000)]{Minka:pjtYdr1C}
Thomas Minka.
\newblock {Estimating a Dirichlet distribution}.
\newblock Technical report, 2000.

\bibitem[Motoyama et~al.(2011)Motoyama, McCoy, Levchenko, Savage, and
  Voelker]{Motoyama_Analysis_2011}
Marti Motoyama, Damon McCoy, Kirill Levchenko, Stefan Savage, and Geoffrey~M.
  Voelker.
\newblock {An analysis of underground forums}.
\newblock In \emph{Proceedings of the 2011 ACM SIGCOMM Internet Measurement
  Conference (IMC-11)}, pages 71--80, 2011.

\bibitem[Newman et~al.(2009)Newman, Asuncion, Smyth, and
  Welling]{Newman:2009uk}
David Newman, Arthur Asuncion, Padhraic Smyth, and Max Welling.
\newblock {Distributed Algorithms for Topic Models}.
\newblock \emph{The Journal of Machine Learning Research}, 10:\penalty0
  1801--1828, 2009.

\bibitem[Sato et~al.(2012)Sato, Kurihara, and Nakagawa]{Sato:2012ke}
Issei Sato, Kenichi Kurihara, and Hiroshi Nakagawa.
\newblock {Practical collapsed variational Bayes inference for hierarchical
  Dirichlet process}.
\newblock In \emph{Proceedings of the 18th ACM SIGKDD International Conference
  on Knowledge Discovery and Data Mining (KDD-12)}, August 2012.

\bibitem[Teh et~al.(2006)Teh, Jordan, Beal, and Blei]{Teh:2006dy}
Yee~Whye Teh, Michael~I Jordan, Matthew~J Beal, and David~M Blei.
\newblock {Hierarchical Dirichlet processes}.
\newblock \emph{Journal of the American Statistical Association}, 101\penalty0
  (476):\penalty0 1566--1581, December 2006.

\bibitem[Teh et~al.(2008)Teh, Kurihara, and Welling]{NIPS2007_763}
Yee~Whye Teh, Kenichi Kurihara, and Max Welling.
\newblock {Collapsed variational inference for HDP}.
\newblock In \emph{Advances in Neural Information Processing Systems 20}, pages
  1481--1488. 2008.

\bibitem[Wallach et~al.(2009{\natexlab{a}})Wallach, Mimno, and
  McCallum]{NIPS2009_0929}
Hanna Wallach, David Mimno, and Andrew McCallum.
\newblock {Rethinking LDA: why priors matter}.
\newblock In \emph{Advances in Neural Information Processing Systems 22}, pages
  1973--1981. 2009{\natexlab{a}}.

\bibitem[Wallach et~al.(2009{\natexlab{b}})Wallach, Murray, Salakhutdinov, and
  Mimno]{icml2009_139}
Hanna Wallach, Iain Murray, Ruslan Salakhutdinov, and David Mimno.
\newblock {Evaluation methods for topic models}.
\newblock In \emph{Proceedings of the 26th International Conference on Machine
  Learning (ICML-09)}, pages 1105--1112, 2009{\natexlab{b}}.

\bibitem[Wang et~al.(2011)Wang, Paisley, and Blei]{Wang:2011ux}
C~Wang, J~Paisley, and D~Blei.
\newblock {Online variational inference for the hierarchical Dirichlet
  process}.
\newblock In \emph{Proceedings of the Fourteenth International Conference on
  Artificial Intelligence and Statistics}, 2011.

\bibitem[Wang and Blei(2012)]{NIPS2012_0208}
Chong Wang and David Blei.
\newblock {Truncation-free online variational inference for Bayesian
  nonparametric models}.
\newblock In \emph{Advances in Neural Information Processing Systems 25}, pages
  422--430. 2012.

\end{thebibliography}

\end{document}